\documentclass[preprint,12pt]{elsarticle}

\usepackage{amssymb}
\usepackage{amsthm}
\usepackage{amssymb, amsmath}

\usepackage[algo2e,linesnumbered,lined,boxed,commentsnumbered]{algorithm2e}
\usepackage{xspace}
\usepackage{algorithmic}
\usepackage{tkz-graph}
\usepackage{multirow}
\usetikzlibrary{shapes.geometric}
\usetikzlibrary{decorations.pathreplacing}

\usepackage{lineno}

\newtheorem{theorem}             {Theorem}
\newtheorem{lemma}      [theorem]{Lemma}
\newtheorem{corollary}  [theorem]{Corollary}

\newtheorem{observation}[theorem]{Observation}

\journal{Theoretical Computer Science}

\begin{document}
% \pagewiselinenumbers
% \switchlinenumbers

\begin{frontmatter}


% \author{Name\corref{cor1}\fnref{label2}}
% \affiliation{organization={},

\title{Runtime Performance of Evolutionary Algorithms for the Chance-constrained Makespan Scheduling Problem}

\author[label1,label2]{Feng Shi}
\author[label1]{Daoyu Huang}
\affiliation[label1]{organization={School of Computer Science and Engineering},
            addressline={Central South University},
            city={Changsha},
            postcode={410083},
            state={Hunan},
            country={P.R. China}}

\affiliation[label2]{organization={Xiangjiang Laboratory},
            city={Changsha},
            postcode={410205},
            state={Hunan},
            country={P.R. China}}

% \author[label3]{Xiankun Yan}
\author{Xiankun Yan\corref{cor1}\fnref{label3}}
\cortext[cor1]{Corresponding Author.}
\ead{xiankun.yan@adelaide.edu.au}

\author[label3]{Frank Neumann}

\affiliation[label3]{organization={Optimisation and Logistics},
            addressline={The University of Adelaide},
            city={Adelaide},
            postcode={5000},
            state={South Australia},
            country={Australia}}

% \affiliation{organization={},%Department and Organization

\begin{abstract}
% \frank{makespan scheduling, local search, etc should be lower case-sovled}
% \frank{makespan scheduling, lower case-sovled}
The makespan scheduling problem is an extensively studied NP-hard problem, and its simplest version is to find an allocation approach for a set of jobs with deterministic processing time to two identical machines such that the makespan is minimized. 
However, in real-life scenarios, the actual processing time of each job may be stochastic under the influence of external factors.
Thus within this paper, we first propose a chance-constrained version of the makespan scheduling problem.
Then we study the theoretical performance of RLS and (1+1) EA for three variants of the chance-constrained makespan scheduling problem.
Within those variants, our theoretical analysis implies that distinct uncertainties influence the behaviors of  the two algorithms. 
Specifically, we separately analyze the expected runtime of the two algorithms to obtain an optimal solution or almost optimal solution to the instances of the three variants. 
In addition, we further investigate the experimental performance of the two algorithms for the three variants.  

\end{abstract}

\begin{graphicalabstract}
\end{graphicalabstract}

\begin{highlights}
\item We consider a chance-constrained version of the makespan scheduling problem.
\item We study the computational complexity of some variants of the problem.
\item We study the theoretical and empirical performance of (1+1) EA and RLS for them. 

\end{highlights}

\begin{keyword}
Chance-constraint \sep Makespan Scheduling problem \sep RLS \sep (1+1) EA \sep Runtime Analysis.



\end{keyword}

\end{frontmatter}

% \linenumbers

\section{Introduction}

\subsection{Background}
To discover the reasons behind the successful applications of evolutionary algorithms in various areas including engineering~\cite{dasgupta2013evolutionary}, logistics~\cite{lin2009integrated}, and economics~\cite{tapia2007applications},
many researchers made efforts to study the theoretical performance of evolutionary algorithms for classical combinatorial optimization problems.
Much of the prior research on the theoretical performance of evolutionary algorithms in combinatorial optimization problems focused on deterministic cases, including the minimum spanning tree~\citep{NeumannW06,neumann2007randomized,kratsch2010fixed,corus2013generalized,witt2014revised,roostapour2020runtime}, vertex cover~\citep{oliveto2008analysis,oliveto2009analysis,friedrich2009analyses,friedrich2010approximating,yu2012approximation,jansen2013approximating,kratsch2013fixed,Pourhassan2015vc,pourhassan2016parameterized,pourhassan2017use}, and knapsack problems~\citep{kumar2005running,he2014theoretical,he2015analysis,wu2016impact,neumann2018runtime,roostapour2018performance,friedrich2020analysis}.
However, such problems in real-world scenarios involve stochastic elements or occur under different dynamic circumstances.
Hence in recent years,
people preferred to explore the theoretical performance of evolutionary algorithms in dynamic and stochastic combinatorial optimization problems~\citep{he2014theoretical,lissovoi2017runtime,ShiSFKN19,neumann2020analysis,ShiNW21,roostapour2022pareto}. The works obtained a series of theoretical results that enhance and expand the comprehension of those algorithms.  

\emph{Chance-constrained optimization problems}~\citep{charnes1959chance,miller1965chance,iwamura1996genetic,poojari2008genetic,li2008chance}, an important class of stochastic optimization problems,
consider that the constraints of the given optimization problem may be influenced by the stochastic components.
For such problems, the goal is to optimize the objective function under the constraints that can be violated only with a small probability. 
A basic technique for solving chance-constrained optimization problems is to convert the stochastic constraints to their respective deterministic equivalents according to the predetermined confidence level.
Recently, some researchers in the area of evolutionary computing began to focus on \emph{chance-constrained optimization problems} and analyze the theoretical performance of evolutionary algorithms for them.
 
The makespan scheduling problem (abbr. MSP)~\citep{BLAZEWICZ198311} is one of the most popular combinatorial optimization problems and has been well-studied particularly in the area of theoretical computer science~\cite{lenstra1990approximation,shmoys1993approximation,witt2005worst}.
The problem considers two identical machines and a set of jobs with deterministic processing time,
and the goal is to allocate the jobs to the two machines such that the makespan (i.e., completion time) to finish all these jobs is minimized (specifically, the makespan is the larger one of the loads on the two machines, where the load on a machine is the sum of the processing time of the jobs allocated to it).\footnote{There are lots of versions of the makespan scheduling problem, but we only consider the simplest version in this paper. Please refer to~\citep{graham1966bounds,sahni1976algorithms,hochbaum1987using,numata1988approximate} for its approximation algorithms, and ~\citep{sahni1976algorithms,numata1988approximate} for its exact algorithms.}
In real-life scenarios, the actual processing time of each job may be stochastic under the influence of external factors. 
% Moreover, the actual processing time of these jobs may be correlated with covariances.
Therefore, a chance-constrained version of MSP, named \emph{chance-constrained makespan scheduling problem} (abbr. CCMSP), is proposed in the paper.
The goal of CCMSP is to minimize a deterministic makespan value and subject to the probability that the actual makespan exceeds the deterministic makespan is no more than an acceptable threshold.
% CCMSP considers two identical machines and several groups of jobs, where the jobs in the same group have the same expected processing time and variance, and their actual processing time are correlated by a covariance if they are allocated to the same machine (i.e., the actual processing time of two jobs in the same group are independent if they are allocated to different machines, and that of the jobs in different groups are independent).
% The goal of CCMSP is to minimize a deterministic makespan value and subject to the probability that the actual makespan exceeds the deterministic makespan is no more than an acceptable threshold.

\subsection{Related work} 

In previous research, a few theoretical results have been obtained about the performance of evolutionary algorithms for MSP and chance-constrained problems.
One of the first runtime analysis of evolutionary algorithms for MSP with two identical machines was carried out by Witt~\cite{witt2005worst}, who studied the approximation and average-case behavior of randomized local search (abbr. RLS) and the (1+1) EA for the problem.
Later Gunia~\cite{gunia2005analysis} extended the results to MSP with a constant number of identical machines.  Sutton and Neumann \cite{sutton2012parameterized} gave the parameterized runtime analysis of RLS and the (1+1) EA for MSP with two identical machines, with respect to three different parameters.
% ~\frank{and Witt-solved}
Neumann and Witt \cite{neumann2015runtime} proposed the dynamic version of MSP with two identical machines and analyzed the performance of RLS and the (1+1) EA for it. 

Xie et al.~\cite{xie2019evolutionary} studied single- and multi-objective evolutionary algorithms for the chance-constrained knapsack problem, where the weights of items are stochastic but the profits are deterministic, and they used Chebyshev's inequality and Chernoff bound separately to estimate the constraint violation probability of a given solution.
Later on, Neumann and Sutton \cite{neumann2019runtime} followed the work of Xie et al.~\cite{xie2019evolutionary} and analyzed several special cases of this problem, getting insight into how the structure of the chance constraint influences the runtime behavior of the (1+1) EA. 
Note that the chance-constrained knapsack problem studied in the above two works does not consider the correlation among the weights of items. 
Thus recently Xie et al.~\cite{xie2021runtime} extended their work and analyzed the expected optimization time of RLS and the (1+1) EA for the problem with correlated uniform weights. 
Additionally, Neumann~\cite{neumann2020optimising} presented the first runtime analysis of multi-objective evolutionary algorithms for chance-constrained submodular functions and showed that the algorithm GSEMO obtains the same worst-case performance guarantees as greedy algorithms~\citep{doerr2020optimization}.
%Doerr et al.~\cite{doerr2020optimization} investigated submodular optimization problems with chance constraints and analyzed the approximation behaviour of greedy algorithms.

Besides these theoretical works, Xie et al.~\cite{xie2020specific} introduced a problem-specific crossover operator and the heavy-tail mutation operator~\citep{doerr2017fast} to evolutionary algorithms to improve their performance for the chance-constrained knapsack problem, and compared their experimental performance.
Moreover, Assimi et al.~\cite{assimi2020evolutionary} studied the experimental performance of a single-objective EA for the dynamic version of the chance-constrained knapsack problem, where each item has an uncertain weight while its profit is deterministic, and the knapsack capacity changes over time every a fixed period of iterations. 
Recently, Neumann et al. ~\cite{NeumannXN22} considered another chance-constrained version of the knapsack problem, where the weights of items are deterministic but the profits are stochastic. Within their investigation, they used the Chebyshev's inequality and Hoeffding bound separately to estimate the constraint violation probability of a given solution.
Besides, they examined three evolutionary algorithms of different types and compared their experimental performance on the problem.

\subsection{Our contribution} 
Within this paper, CCMSP considers two identical machines and several groups of jobs, where the jobs in the same group have the same expected processing time and variance, and a covariance that quantifies the shared variation in their actual processing times if they are allocated to the same machine (i.e., the actual processing time of two jobs in the same group are independent if they are allocated to different machines, and that of the jobs in different groups are independent).
We investigate the expected runtime of RLS and the (1+1) EA for CCMSP. 
More specifically, we consider two special variants of CCMSP:
CCMSP-1, all jobs have the same expected processing time and variance, and all groups have the same covariance and the same \emph{even} size;
CCMSP-2, all jobs have the same expected processing time and variance, the groups have the same covariance but different sizes. 
For CCMSP-1, we prove that it is polynomial-time solvable by showing that RLS and the (1+1) EA can obtain an optimal solution to any instance $I_1$ of CCMSP-1 in expected runtime $O(n^2/m)$ and $O((k+m)n^2)$, respectively, where $n$ and $k$ are the numbers of jobs and groups considered in $I_1$, and $m = n/k$.
For CCMSP-2, the size difference among groups makes the discussion complicated, thus two simplified variants of CCMSP-2 named CCMSP-2$^+$ and CCMSP-2$^-$ are proposed: CCMSP-2$^+$ assumes that the sum of the variances and covariances of the jobs allocated to the same machine cannot be greater than the expected processing time of a job,
no matter how many jobs are allocated to the machine;
CCMSP-2$^-$ assumes that the sum of the variances and expected processing time of the jobs cannot be greater than their covariances.
We prove that CCMSP-2$^+$ is NP-hard and that RLS can obtain an optimal solution to any instance $I_2^+$ of CCMSP-2$^+$ in expected polynomial time if the total number of jobs considered in $I_2^+$ is odd;
otherwise (i.e., the total number of jobs considered is even), a solution close to the optimal one to $I_2^+$. 
We show that CCMSP-2$^-$ is NP-hard as well and analyze the approximation ratio of the local optimal solution obtained by  RLS.

In addition, we study the experimental performance of the two algorithms.
For each variant, we specify a parameter-triple $(k,n,c)$, in which $c$ has 3 settings, and each of $k$ and $n$ has 6 settings (i.e., there are 108 instances of each variant that are constructed). Then we study the number of iterations that the algorithms RLS and (1+1) EA require to get an optimal or locally optimal solution for each instance and find that the experimental results coincide with the theoretical ones obtained for them. 
Specifically, for each constructed instance $I$ of CCMSP-2$^+$ with the even number of jobs, a corresponding instance $I'$ of CCMSP-2$^+$ with the odd number of jobs is constructed by introducing an extra job to an arbitrary group of $I$, for ease of comparison. The experimental result shows that the tiny difference between the two instances causes a huge gap between the experimental performance of RLS on them, which coincides with the computational hardness of CCMSP-2$^+$ with the odd number of jobs mentioned above.

This paper extends and refines the conference version~\citep{DBLP:conf/ppsn/ShiYN22}.
Firstly, the study of CCMSP-2$^-$ and the experimental analysis of the algorithms RLS and (1+1) EA are newly added (Subsection 5.3 and Section 6). 
The supplementary research provides compelling evidence of the critical need to address the influence of stochastic elements.
The results indicate varied levels of variance lead to discrepancies in algorithm performance.
Secondly, the theoretical performance of RLS for the instances of CCMSP-2$^+$ with the even number of jobs is rewritten (Subsection 5.2).
Finally, more details and illustrations for the theoretical performance of the two algorithms are included (e.g., Section 4 and Section 5). 

The paper is structured as follows. 
Section 2 introduces the related definitions and the formulation of CCMSP. 
Section 3 presents the two considered algorithms, namely, RLS and the (1+1) EA, and Section 4 analyzes their theoretical performance to obtain an optimal solution to the instances of CCMSP-1. 
Section 5 considers the theoretical performance of RLS for solving the instances of CCMSP-2$^+$ and CCMSP-2$^-$. 
The experimental performance of RLS and the (1+1) EA for CCMSP-1, CCMSP-2$^+$ and CCMSP-2$^-$ is studied in Section 6.
Section 7 is used to conclude this work.

\section{Preliminaries}

The notation $[x, y]$ ($x$ and $y$ are two integers with $x \le y$) denotes the set containing all integers ranging from $x$ to $y$. 
Let $[x, y) = [x, y] \setminus \{y\}$ and $(x, y] = [x, y] \setminus \{x\}$.
Consider two identical machines $M_0$ and $M_1$, and $k$ groups of jobs, where each group $G_i$ ($i \in [1,k]$) has $m_i$ jobs. 
That is, there are $n = \sum_{i=1}^k m_i$ jobs in total.
W.l.o.g., assume that the sizes of the $k$ groups are in non-decreasing order (i.e., $m_1 \le m_2 \le \ldots \le m_k$).
The $j$-th job in group $G_i$ ($i \in [1,k]$ and $j \in [1, m_i]$), denoted by $b_{ij}$, has actual processing time $p_{ij}$ with expected value $E[p_{ij}] = a_{ij} > 0$ and variance $\sigma^2_{ij} \ge 0$. 
Additionally, for any two jobs of the same group $G_i$, if they are allocated to the same machine, then their actual processing time are correlated with each other by a covariance $c_i \ge 0$; otherwise, independent.
For CCMSP, the objective is to compute an allocation of the $n$ jobs to the two machines that minimizes the makespan $M$ such that the probabilities of the \emph{loads} on $M_0$ and $M_1$ exceeding $M$ are no more than a given acceptable threshold $0 \le \gamma < 1$, where the load on $M_t$ ($t \in [0,1]$) is the sum of the actual processing time of the jobs allocated to $M_t$.

An allocation (simply called solution in the remaining text) $x$ to an instance of CCMSP, is encoded as a bit-string with length $n$ in this paper, $x = (x_{11}...x_{1m_1}...x_{i1}...x_{im_i}...x_{k1}...x_{km_k}) \in \{0,1\} ^n$, 
where the job $b_{ij}$ is allocated to machine $M_0$ if $x_{ij} = 0$;
otherwise, allocated to machine $M_1$.
For ease of notation, denote by $M_t(x)$ the set of jobs allocated to $M_t$ ($t \in [0,1]$), with respect to $x$, and by $l_t(x) = \sum_{b_{ij} \in M_t(x)} p_{ij}$ the \emph{load} on $M_t$. 
Let $\alpha_i(x) = |M_0(x) \cap G_i|$ and $\beta_i(x) = |M_1(x) \cap G_i|$ for all $i \in [1,k]$.
The CCMSP can be formulated as:
\begin{align*}
    &\textbf{Minimize}  \qquad  M \\
    &\textbf{Subject to} \quad  \Pr(l_t(x) > M) \leq \gamma \ \ \textrm{for all} \ \ t \in [0,1]. 
\end{align*}

%\begin{align*}
    %&f_1(x) = \sum_{i=1}^{K}\sum_{j=1}^{m} p_{ij} \cdot (1-x_{ij}) \ \ \textrm{and} \\
    %&f_2(x) = \sum_{i=1}^{K}\sum_{j=1}^{m} p_{ij} \cdot x_{ij}.
%\end{align*}

Since $E[p_{ij}] = a_{ij}$, the expected value of $l_t(x)$ is
$E[l_t(x)] = \sum_{b_{ij} \in M_t(x)} a_{ij}$.
%\begin{align*}
    %&E[f_1(x)] = \sum_{i=1}^{k}\sum_{j=1}^{m} a_{ij} \cdot (1-x_{ij});\\
    %&E[f_2(x)] = \sum_{i=1}^{k}\sum_{j=1}^{m} a_{ij} \cdot x_{ij}.
%\end{align*}
Considering the variance $\sigma^2_{ij}$ of each job $b_{ij}$ and the covariance among the jobs of the same group if they are allocated to the same machine, the variance of $l_t(x)$ is 
$$var[l_t(x)] = \sum_{b_{ij} \in M_t(x)} \sigma^2_{ij} + cov[l_t(x)],$$
where $ cov[l_t(x)] = \sum_{i=1}^{k} 2c_i \times \binom{|M_t(x) \cap G_i|}{2}.$
Note that $\binom{|M_t(x) \cap G_i|}{2} = 0$ if $0 \le |M_t(x) \cap G_i| \le 1$. 

% \begin{align*}
%     Var[f_1(x)] = d \times |M_1(x)| + 2c\sum_{i=1}^{k} \binom{|M_1(x) \cap G_i|}{2}
%     %Var[f_1(x)] &= d \sum_{i=1}^{k}\sum_{j=1}^{m}(1-x_{ij})\\ 
%     %    &+ 2c\sum_{i=1}^{k}\sum_{1\leq j< j' \leq m} [(1-x_{ij}) \cdot (1-x_{ij'})];
% \end{align*}
% \begin{align*}
%     Var[f_2(x)] = d \times |M_2(x)| + 2c\sum_{i=1}^{k} \binom{|M_2(x) \cap G_i|}{2}
%     %Var[f_2(x)] = d \sum_{i=1}^{k}\sum_{j=1}^{m}x_{ij} 
%      %   + 2c\sum_{i=1}^{k}\sum_{1\leq j< j' \leq m} x_{ij} \cdot x_{ij'}.
% \end{align*}

For the probability $Pr(l_t(x) > M)$ with $t \in [0,1]$, following the work of~\cite{xie2019evolutionary,xie2021runtime}, we use the one-sided Chebyshev's Inequality (cf. Theorem~\ref{thm:One-sided Chebyshev's Inequality}) to construct a usable surrogate of the chance-constraint.

\begin{theorem}
\label{thm:One-sided Chebyshev's Inequality}
(one-sided Chebyshev's Inequality). \textsl{Let $X$ be a random variable with expected value $E[X]$ and variance $var[X]$. 
Then for any $\Delta \in \mathbb{R}^+$,} $$\Pr(X > E[X] + \Delta) \leq \frac{var[X]}{var[X] + \Delta^2}.$$
\end{theorem}

First, we define $l'_t(x):= \sqrt{\frac{(1-\gamma)}{\gamma} var[l_t(x)]} + E[l_t(x)] $. 
Then, by the one-sided Chebyshev's Inequality, upper bounding the probability of the actual makespan exceeding $M$ by $\gamma$ indicates that for all $t \in [0,1]$,
$$
   \Pr(l_t(x) >  M) \leq\frac{var[l_t(x)]}{var[l_t(x)]+(M-E[l_t(x)])^2} \leq \gamma \iff l'_t(x) \leq M.
$$
Note that as the value of $\Delta$ considered in the one-sided Chebyshev's Inequality is positive, thus $M - E[l_t(x)]$ is greater than 0 by default, and the left right double arrow (i.e., the if and only if relationship) given in the above expression holds.  
Summarizing the above analysis gets that $\max \{\Pr(l_0(x) >  M), \Pr(l_1(x) >  M)\} \leq \gamma$ holds if and only if 
$$L(x) =  \max \{l_0'(x),l_1'(x)\} \leq M.$$
In other words, $L(x)$ is the tight lower bound for the value of $M$, if using the surrogate of the chance constraint by the One-sided Chebyshev's Inequality. 
Consequently, $l'_t(x)$ can be treated as a \emph{new} measure for the load on the machine $M_t$, and the goal of CCMSP is simplified to minimize the value of $L(x)$. 
Let $t(x) = \arg_{t \in [0,1]} \max \{l_0'(x),l_1'(x)\}$.

% In order to  get the lower bound of $M$, we define $M'$ such that $max\{Pr(f_1(x) > M'),Pr(f_2(x) > M')\} = \alpha$. Then, the surrogate function $\beta$ is defined as 
% \begin{equation}
%     \beta(x) = max\{f_1'(x),f_2'(x)\},
% \end{equation}
% where $f_1'(x)$ and $f_2'(x)$ are the actual load bounded by the Chebyshev's inequality for $M_1$ and $M_2$ respectively. 

% When the value of $\beta(x)$ decrease, $M$ is minimized, the fitness function for RLS and (1+1) EA algorithm is employed by $$ F(x) := \beta(x).$$
%Besides, in our investigations, we denote by $d(x) = |f_1'(x) - f_2'(x)|,$ the discrepancy of the solution x.

It can be derived that CCMSP is NP-hard, as CCMSP is the same as MSP if $\sigma^2_{ij} = c_i = 0$ for all $i \in [1,k]$ and $j \in [1,m_i]$, and $\gamma = 0$. 
Witt~\cite{witt2005worst} showed that the classical RLS and the (1+1) EA for MSP may get trapped into a bad local optima caused by the different processing time of the jobs considered, such that they require $O(n^{\Omega(n)})$ runtime to escape it. 
Observe that the same issue also applies to the two algorithms for CCMSP. 
Thus within the paper, we only study the two specific variants of CCMSP given below, in which the jobs considered have the same expected processing time.

\smallskip

{\bf CCMSP-1}.
All the $n$ jobs considered have the same expected processing time $a_{ij}=a > 0$ and variance $\sigma^2_{ij}=d \ge 0$, and the $k$ groups have the same covariance $c \ge 0$ and size $m > 0$ (i.e., $c_i = c$ and $m_i = m$ for all $i \in [1,k]$).
Moreover, $m$ is even.

\smallskip

{\bf CCMSP-2}. 
All the $n$ jobs considered have the same expected processing time $a_{ij}=a > 0$ and variance $\sigma^2_{ij}=d > 0$, and the $k$ groups have the same covariances $c \ge 0$ (i.e., $c_i = c$ for all $i \in [1,k]$).
However, the $k$ groups may have different sizes (may be even or odd).

\smallskip

Given an instance $I$ of CCMSP-1 or CCMSP-2 and an arbitrary solution $x$ to $I$, if $||M_0(x)| - |M_1(x)|| \le 1$ (i.e., $|M_0(x)| = |M_1(x)|$ if $n$ is even), then $x$ is an \emph{equal solution}; if $||M_0(x)| - |M_1(x)|| \le 1$, and $|\alpha_i(x) - \beta_i(x)| \le 1$ for all $i \in [1,k]$ (i.e., $\alpha_i(x) = \beta_i(x)$ if $m_i$ is even), then $x$ is a \emph{balanced solution}. 
%That is, a balanced-solution is an equal solution, but not vice versa.

\section{Algorithms}

Within the paper, we study the performance of the classical randomized local search (abbr. RLS, given as Algorithm~\ref{alg:RLS}) and (1+1) EA (given as Algorithm~\ref{alg:1+1}) for the two variants of CCMSP. 
Both algorithms are primarily used for the theoretical analysis of variable combinatorial problems due to their robust yet simple mechanisms~\cite{neumann2006minimum,witt2014revised,oliveto2009analysis,witt2005worst}.
The two algorithms run in a similar way, randomly generating a new offspring based on the current maintained solution and replacing it if the offspring is not worse than it in terms of its fitness. 
The difference between the two algorithms is the way to generate offspring.
RLS chooses one bit of the current solution uniformly at random and flips it with a probability of 1/2; otherwise, it chooses two bits of the current solution uniformly at random and flips them.
In terms of the (1+1) EA,
it flips each bit of the current maintained solution with probability $1/n$.

The fitness function considered in the two algorithms is the natural one, 
$$f(x) = L(x) = \max \{l_0'(x),l_1'(x)\}.$$

\begin{algorithm2e}[t]
\caption{RLS}
\label{alg:RLS}
\SetAlgoSkip{tinyskip}
    choose $x \in \{0,1\}^n$ uniformly at random\;
    \While{stopping criterion not met}{
        choose $b \in \{0,1\}$ uniformly at random\;
        \If {$b=0$}{
        $y$ $\gets$ flip one bit of $x$ chosen uniformly at random\;}
        \Else{choose $(i,j) \in  \{(k,l)| 1 \le k < l \le n\}$ uniformly at random\;
        $y$ $\gets$ flip the $i$-th and $j$-th bits of $x$\;}
        \If {$f(y) \leq f(x)$}{$x \gets y$\;}
    }
\end{algorithm2e}

%\vspace*{0mm}

\begin{algorithm2e}[t]
\caption{(1+1) EA}
\label{alg:1+1}
\SetAlgoSkip{tinyskip}
    choose $x \in \{0,1\}^n$ uniformly at random\;
    \While{stopping criterion not met}{
        $y$ $\gets$ flip each bit of $x$ independently with probability $1/n$\;
        \If {$f(y) \leq f(x)$}{$x \gets y$\;}
    }
\end{algorithm2e}

\section{Theoretical Performance for CCMSP-1}

The section starts with a lemma that will be used throughout this section.

\begin{lemma}
\label{lem:1}
$\binom{\lfloor \frac{x+y}{2} \rfloor}{2} + \binom{\lceil \frac{x+y}{2} \rceil}{2}  \le \binom{x}{2} + \binom{y}{2} \le \binom{x + y}{2}$ holds for any two natural numbers $x$ and $y$.
\end{lemma}
\begin{proof}
    The binomial coefficient $\binom{x}{2}$ equals $x(x-1)/2$. 
    Then we have
    \begin{align*}
        &\binom{x}{2} + \binom{y}{2} - \left(\binom{\lfloor \frac{x+y}{2} \rfloor}{2} + \binom{\lceil \frac{x+y}{2} \rceil}{2} \right) \\
        & = \frac{x(x-1)}{2} + \frac{y(y-1)}{2} - \left( \frac{\left\lfloor\frac{x+y}{2}\right\rfloor (\left\lfloor\frac{x+y}{2}\right\rfloor-1)}{2} + \frac{\left \lceil \frac{x+y}{2} \right\rceil (\left\lceil \frac{x+y}{2} \right\rceil-1))}{2}   \right) \\ 
        & = \frac{x^2+y^2 - \left({\left\lfloor\frac{x+y}{2}\right\rfloor}^2 + {\left\lceil \frac{x+y}{2} \right\rceil}^2\right)}{2}.
    \end{align*}
    If $x+y$ is even, then the above expression equals $(x-y)^2/2 \ge 0$; otherwise, it equals $\frac{(x-y)^2-1}{2} \ge 0$.
    By the similar reasoning, we can derive that $\binom{x + y}{2}-\left(\binom{x}{2} + \binom{y}{2}\right)\geq 0$. The proof is done.  
\end{proof}

% \begin{observation}
% \label{obs:1}
% $\binom{x}{2} \ge \binom{x - y}{2} + \binom{y}{2} \ge \binom{\lceil x/2 \rceil}{2} + \binom{\lfloor x/2 \rfloor}{2}$ holds for any $x \ge 2$ and $x-1 \ge y \ge 1$.
% \end{observation}

Consider an instance $I_1 = (a,c,d,\gamma,k,m)$ of CCMSP-1 and a solution $x$ to $I_1$. 
% \xiankun{reviwer-4 why $\delta_i(x)$ is non-negative}
% \xiankun{It doesn't matter whether $\delta_i(x)$ is non-negative due to the defined $\alpha_i(x)$ and $\beta_i(x)$.}
As the groups considered in $I_1$ have the same size $m$, there is a variable $-\frac{m}{2} \le \delta_i(x) \le \frac{m}{2}$ such that $\alpha_i(x) = \frac{m}{2} + \delta_i(x)$ and $\beta_i(x) = \frac{m}{2} - \delta_i(x)$ for any $i \in [1,k]$. Then with the help of Lemma~\ref{lem:1}, we have the following lemma to investigate the relationship between the loads on the machines $M_0$ and $M_1$.

\begin{lemma}
\label{lem:inequal number}
For any solution $x$ to the instance $I_1 = (a,c,d,\gamma,k,m)$ of CCMSP-1, if $|M_0(x)| > |M_1(x)|$ (resp., $|M_1(x)| > |M_0(x)|$) then $l'_0(x) > l'_1(x)$ (resp., $l'_1(x) > l'_0(x)$); if $x$ is an equal solution then $L(x) = l'_0(x) = l'_1(x)$.
\end{lemma}

\begin{proof}
First of all, we show the claim that if $|M_0(x)| > |M_1(x)|$ then $l'_0(x) > l'_1(x)$. 
Recall that all jobs have the same expected processing time $a$ and variance $d$, thus $E[l_0(x)] > E[l_1(x)]$ and  
\begin{equation}
\label{equ:inequal number 1}
 |M_0(x)| d >  |M_1(x)| d.
\end{equation}
Now we show that $cov[l_0(x)] > cov[l_1(x)]$, i.e., 
\begin{equation}
\label{equ:inequal number 2} 
\binom{\alpha_1}{2} + \ldots + \binom{\alpha_k}{2} > \binom{\beta_1}{2} + \ldots + \binom{\beta_k}{2}.
\end{equation}
For any specific $i \in [1,k]$, $\binom{\alpha_i}{2} - \binom{\beta_i}{2} = \binom{m/2 + \delta_i}{2} - \binom{m/2 - \delta_i}{2}
= (m-1)\delta_i$.
% \begin{eqnarray*}
% \label{equ:inequal number 3}
% \begin{aligned}
% \binom{\alpha_i}{2} - \binom{\beta_i}{2} &= \binom{m/2 + \delta_i}{2} - \binom{m/2 - \delta_i}{2}
% = (m-1)\delta_i.
% \end{aligned}
% \end{eqnarray*}
Then
\begin{eqnarray}
\label{equ:inequal number 4}
\begin{aligned}
\sum_{i=1}^k \left(\binom{\alpha_i}{2} - \binom{\beta_i}{2}\right) &=  \sum_{i=1}^k (m-1)\delta_i = (m-1) \sum_{i=1}^k\delta_i \\
&= \frac{m-1}{2} \left(\sum_{i=1}^{k} \alpha_i(x) - \sum_{i=1}^{k}\beta_i(x) \right) \\
&= \frac{m-1}{2} \left( |M_0(x)| -  |M_1(x)| \right) > 0.
\end{aligned}
\end{eqnarray}
Thus Inequality~(\ref{equ:inequal number 2}) holds.
Combining Inequalities~(\ref{equ:inequal number 1}) and~(\ref{equ:inequal number 2}) gives that $var[l_0(x)] > var[l_1(x)]$. As the considered threshold $\gamma$ is a fixed constant, summarizing the above discussion shows the claim. 
Because of symmetry, we can also derive that if $|M_0(x)| < |M_1(x)|$ then $l'_0(x) < l'_1(x)$. 

Now we show that if $|M_0(x)| = |M_1(x)|$ then $l'_0(x) = l'_1(x)$.
Observe that $E[l_0(x)]= E[l_1(x)]$, $\sum_{b_{ij} \in M_0(x)} d = \sum_{b_{ij} \in M_1(x)} d$ and $\sum_{i=1}^k \delta_i = 0$.
Combining $\sum_{i=1}^k \delta_i = 0$ and Expression~(\ref{equ:inequal number 4}) gives that 
\begin{equation*}
\label{equ:equal number 1} 
\binom{\alpha_1}{2} + \ldots + \binom{\alpha_k}{2} = \binom{\beta_1}{2} + \ldots + \binom{\beta_k}{2}.
\end{equation*}
Therefore, $var[l_0(x)] = var[l_1(x)]$ and $l'_0(x) = l'_1(x) = L(x)$.
\end{proof}

Based on Lemma~\ref{lem:inequal number}, we have the property of the optimal solution to the instance $I_1$ of CCMSP-1, which is introduced in the following Lemma~\ref{lem:equal number}.

\begin{lemma}
\label{lem:equal number}
For any solution $x$ to the instance $I_1 = (a,c,d,\gamma,k,m)$ of CCMSP-1, if $x$ is a balanced solution then $L(x) = l'_0(x) = l'_1(x)$ has the minimum value; more specifically, $x$ is an optimal solution to $I_1$ if and only if $x$ is a balanced solution to $I_1$.
\end{lemma}

% \xiankun{reviewer 4 - Lemma 4: the first claim does not follow from $A > B$, as both $A$ and $B$ are lower bounds to their respective $L(x)$. Let $L_A(x)$ is $L(x)$ in the first case, $L_B(x)$ is $L(x)$ in the second case, then it can be that $L_A(x) >= A$, $L_B(x) >= B$, $A > B$, but $L_B(x) > L_A(x)$.}

% \xiankun{The provment is ok. Because we only compared the minimum value of $L(x)$ under different cases.The minimum values are fixed.}

\begin{proof}
To show that $L(x) = l'_0(x) = l'_1(x)$ has the minimum value if $\alpha_i = \beta_i = m/2$ for all $i \in [1,k]$, we first prove the claim that the minimum value of $L(x)$ if $|M_0(x)| \neq |M_1(x)|$ is larger than that of $L(x)$ if $|M_0(x)| = |M_1(x)|$.

For the minimum value of $L(x)$ if $|M_0(x)| \neq |M_1(x)|$, w.l.o.g., we assume that $|M_0(x)| > |M_1(x)|$, i.e., $\sum_{i=1}^k \delta_i \ge 1$. 
Besides we have $\sum_{i=1}^k \delta_i^2 \ge 1$.
By Lemma~\ref{lem:inequal number}, $L(x) = l'_0(x) > l'_1(x)$, and
\begin{eqnarray*}
\label{equ:equal number 2}
cov[l_0(x)] / c &=& \sum_{i=1}^k \alpha_i (\alpha_i -1) = \sum_{i=1}^k (m/2 + \delta_i)^2 - \sum_{i=1}^k (m/2 + \delta_i) \\
&=& km^2/4 + (m-1) \sum_{i=1}^k \delta_i +  \sum_{i=1}^k \delta_i^2 - n/2 \\
&\ge& km^2/4 + (m-1) +  1 - n/2 \\
&\ge& km^2/4 + m - n/2.
\end{eqnarray*}
Thus if $|M_0(x)| > |M_1(x)|$ then $var[l_0(x)] \ge d(n/2 + 1) + c(km^2/4 + m - n/2)$ and
\begin{align*}
    L(x) &= l'_0(x) \\
    &\ge \sqrt{\frac{(1-\gamma)}{\gamma} \left( d(n/2 + 1) + c(km^2/4 + m - n/2) \right)} + (n/2 + 1)a = A.
\end{align*}
$$$$
That implies the minimum value of $L(x)$ if $|M_0(x)| > |M_1(x)|$ is at least $A$.

For the minimum value of $L(x)$ if $|M_0(x)| = |M_1(x)|$, we have that $\sum_{i=1}^k \delta_i = 0$, $L(x) = l'_0(x) = l'_1(x)$ due to Lemma~\ref{lem:inequal number}, and
\begin{eqnarray}
\label{equ:equal number 3}
\begin{aligned}
cov[l_0(x)] / c &= km^2/4 + (m-1) \sum_{i=1}^k \delta_i +  \sum_{i=1}^k \delta_i^2 - n/2 \\
&\ge km^2/4 + 0 +  0 - n/2 = km^2/4 - n/2.
\end{aligned}
\end{eqnarray}
Thus if $|M_0(x)| = |M_1(x)|$ then $var[l_0(x)] \ge dn/2 + c(km^2/4 - n/2)$ and
$$L(x) = l'_0(x) \ge \sqrt{\frac{(1-\gamma)}{\gamma}\left(dn/2 + c(km^2/4 - n/2)\right)} + na/2 = B.$$
By Inequality~(\ref{equ:equal number 3}), under the case that $|M_0(x)| = |M_1(x)|$ and $\delta_i = 0$ for all $i \in [1,k]$, we have $\sum_{i=1}^k \delta_i^2 = 0$. Meanwhile, $L(x)$ exactly gets its minimum value $B$.
Note that $A > B$, thus the first claim is correct.
Besides, since $B$ is the lowest value of $L(x)$,
it implies that $x$ is a balanced solution to $I_1$ if and only if $x$ is optimal to $I_1$. 
The proof is finished.  
\end{proof}

\begin{figure*}[t]
  \centering     \footnotesize 
 % \subfloat[{}]{\label{}
    \begin{tikzpicture}[every node/.style={circle,black,draw,fill=black,inner sep=1pt}, scale=.55]
        %\node[label=above:$r$] (0) at (0, 1.5) {};
        %\node[label=left:$p_1$] (1) at (-1, 0) {};
        %\node[label=below:$v_1$] (2) at (0, -1.5) {};
        %\node[label=right:$j$,scale =0] (3) at (-0.6, -0.5) {};
        %\node[label=right:$i$,scale =0] (4) at (0, 0) {};
        %\node[scale =0.01] (1) at (1.2, 2.2) {};
        %\node[scale =0.01] (2) at (3.8, 2.2) {};

        \node[label=right:{\rotatebox{90}{initial solution}},scale =0] (1) at (-2.6, 2.5) {};
        \node[label=right:{\rotatebox{90}{equal solution}},scale =0] (1) at (8.5, 2.5) {};
        \node[label=right:{\rotatebox{90}{balanced solution}},scale =0] (1) at (19.2, 2.5) {};
        \node[label=right: \tiny{flip a 0-bit},scale =0] (1) at (0.2, 4.6) {};
        \node[label=right: \tiny{flip a 1-bit},scale =0] (1) at (0.2, 3.6) {};
        \node[label=right: \tiny{flip two 0-bits},scale =0] (1) at (0.2, 2.6) {};
        \node[label=right: \tiny{flip a 0-bit and a 1-bit},scale =0] (1) at (0.2, 1.6) {};
        \node[label=right: \tiny{flip two 1-bits},scale =0] (1) at (0.2, 0.6) {};
        
        \draw[decorate,decoration={brace,mirror},ultra thick] (-0.5,-0.2) -- (10.4,-0.2);
        \node[label=right:\scriptsize{Phase-1: expected runtime $O(n)$},scale =0] (1) at (1.3, -0.9) {};
        \draw[decorate,decoration={brace,mirror},ultra thick] (10.6,-0.2) -- (21.5,-0.2);
        \node[label=right:\scriptsize{Phase-2: expected runtime {$O(n^2/m)=O(nk)$}},scale =0] (1) at (11, -0.9) {};

    \begin{scope}[every path/.style={ultra thin}]
    
        \draw[rounded corners] (-1,0) rectangle (0,5);
        %\draw (0.5,0.6) circle (0.2);
        %\draw (0.5,2.1) node [scale =2.4] {};

        \draw[->,thick] (0.2, 4.3) to (4.8, 4.3);
        \draw[->,thick] (0.2, 3.3) to (4.8, 3.3);
        \draw[->,thick] (0.2, 2.3) to (4.8, 2.3);
        \draw[->,thick] (0.2, 1.3) to (4.8, 1.3);
        \draw[->,thick] (0.2, 0.3) to (4.8, 0.3);

        \draw[->,dashed,thick] (5.2, 2.3) to (9.7, 2.3);
        \node[label=right: \tiny{accepted solution has},scale =0] (1) at (5.2, 3.7) {};
        \node[label=right: \tiny{$||M_0(x)| - |M_1(x)||$},scale =0] (1) at (5.2, 3.1) {};
        \node[label=right: \tiny{non-increase},scale =0] (1) at (5.2, 2.6) {};

        \draw[rounded corners] (10,0) rectangle (11,5);

        \draw[->,thick] (11.2, 2.3) to (16, 2.3);
        \node[label=right: \tiny{flip a 0-bit and a 1-bit},scale =0] (1) at (11.2, 2.6) {};

        %\draw (12.5,2.1) node [scale =2.4] {};

        \draw[->,dashed,thick] (16.3, 2.3) to (21, 2.3);
        \node[label=right: \tiny{accepted solution has},scale =0] (1) at (16.3, 3.8) {};
        \node[label=right: \tiny{$\sum_{i=1}^k |\alpha_i(x) - \beta_i(x)|$},scale =0] (1) at (16.3, 3.2) {};
        \node[label=right: \tiny{non-increase},scale =0] (1) at (16.3, 2.6) {};

        \draw[rounded corners] (21.2,0) rectangle (22.2,5);
    \end{scope}
    \end{tikzpicture}
    \vspace*{-30mm}
\caption{An illustration for the proof of Theorem~\ref{thm:RLS CCMSP-1}.}
\label{fig:illustration of RLS CCMSP-1}
\end{figure*}
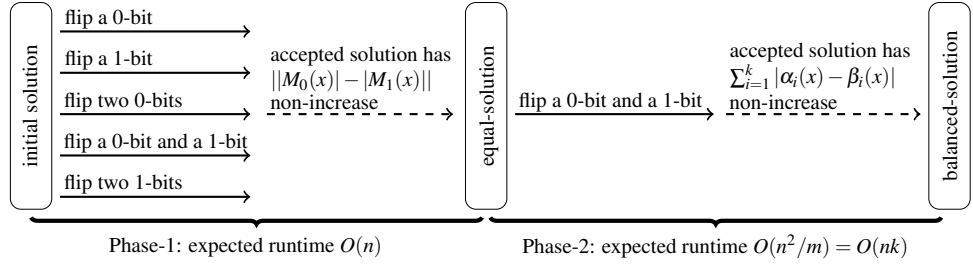

Now we are ready to analyze the expected runtime of RLS and the (1+1) EA to solve CCMSP-1.

\begin{theorem}
\label{thm:RLS CCMSP-1}
The expected runtime of RLS to obtain an optimal solution to the instance $I_1 = (a,c,d,\gamma,k,m)$ of CCMSP-1 is $O(n^2/m) = O(kn)$.
\end{theorem}

\begin{proof}
Lemma~\ref{lem:equal number} shows that any optimal solution to $I_1$ is a balanced one.
Let $x_0$ be the initial solution maintained by RLS that is assumed to have $|M_0(x_0)| > |M_1(x_0)|$.
Thus $L(x) = l'_0(x_0) > l'_1(x_0)$ by Lemma~\ref{lem:inequal number} and $|M_0(x_0)| - |M_1(x_0)| \ge 2$ as $n = mk$ is even.
The optimization process of RLS based on $x_0$ discussed below is divided into two phases, illustrated in Fig.~\ref{fig:illustration of RLS CCMSP-1}. 

{\bf Phase-1}. Obtaining the first equal solution $x_1$ based on $x_0$.
Five possible cases for the mutation of RLS on $x_0$ are listed below, obtaining an offspring $x'_0$ of $x_0$.

{\bf Case (1)}. Flipping a 0-bit in $x_0$ (i.e., $|M_0(x'_0)| = |M_0(x_0)| -1$). 
Observe that $E[l_0(x_0)] > E[l_0(x'_0)]$, $var[l_0(x_0)] > var[l_0(x'_0)]$ and $l'_0(x_0) > l'_0(x'_0)$, since a job on $M_0$ with respect to $x_0$ is moved to $M_1$.
As  $|M_0(x_0)| - |M_1(x_0)| \ge 2$ we have $|M_0(x'_0)| \ge |M_1(x'_0)|$.
Thus $L(x'_0) = l'_0(x'_0)$ by Lemma~\ref{lem:inequal number}, $L(x'_0) = l'_0(x'_0) < l'_0(x_0) = L(x_0)$, and $x'_0$ can be accepted by RLS.

{\bf Case (2)}. Flipping a 1-bit in $x_0$ (i.e., $|M_0(x'_0)| = |M_0(x_0)| +1$).
Since a job on $M_1$ with respect to $x_0$ is moved to $M_0$, $E[l_0(x'_0)] > E[l_0(x_0)]$, $var[l_0(x'_0)] > var[l_0(x_0)]$, and $l'_0(x'_0) > l'_0(x_0)$.
Similarly, $l'_1(x_0) > l'_1(x'_0)$.
Since $l'_0(x_0) > l'_1(x_0)$, $L(x'_0) = l'_0(x'_0) > l'_0(x_0) = L(x_0)$ and $x'_0$ cannot be accepted by RLS.

{\bf Case (3)}. Flipping two 0-bits in $x_0$ (i.e., $|M_0(x'_0)| = |M_0(x_0)| -2$). 
If $|M_0(x'_0)| \ge |M_1(x'_0)|$, then using the reasoning for Case (1) gets that $L(x'_0) \le L(x_0)$ and $x'_0$ can be accepted by RLS.
If $|M_0(x'_0)| < |M_1(x'_0)|$ then by the facts that $|M_0(x_0)| - |M_1(x_0)| \ge 2$ and $n$ is even, $|M_0(x_0)| - |M_1(x_0)| = 2$, $|M_1(x'_0)| - |M_0(x'_0)| = 2$ and $|M_0(x_0)| = |M_1(x'_0)|$. 
By Lemma~\ref{lem:inequal number}, we have
$$L(x_0) - L(x'_0) = l'_0(x_0) - l'_1(x'_0) = \sqrt{\frac{1-\gamma}{\gamma}} \left(\sqrt{var[l_0(x_0)]} - \sqrt{var[l_1(x'_0)]} \right).$$
Recall that $|M_0(x_0)| = |M_1(x'_0)|$, thus
$$\sqrt{var[l_0(x_0)]} \ge \sqrt{var[l_1(x'_0)]} \ \Longleftrightarrow \ cov[l_0(x_0)] \ge cov[l_1(x'_0)].$$
Consequently, $x'_0$ can be accepted by RLS if and only if $cov[l_0(x_0)] \ge cov[l_1(x'_0)]$.

{\bf Case (4)}. Flipping a 0-bit and a 1-bit in $x_0$ (i.e., $|M_0(x'_0)| = |M_0(x_0)|$). 
Observe $|M_0(x_0)| = |M_0(x'_0)| \ge |M_1(x'_0)|$.
By Lemma~\ref{lem:inequal number}, 
$$L(x_0) - L(x'_0) = l'_0(x_0) - l'_0(x'_0) = \sqrt{\frac{1-\gamma}{\gamma}} \left(\sqrt{var[l_0(x_0)]} - \sqrt{var[l_0(x'_0)]} \right),$$
$$ \textrm{and} \ \sqrt{var[l_0(x_0)]} \ge \sqrt{var[l_0(x'_0)]} \ \Longleftrightarrow \ cov[l_0(x_0)] \ge cov[l_0(x'_0)].$$ 
Consequently, $x'_0$ can be accepted by RLS if and only if $cov[l_0(x_0)] \ge cov[l_0(x'_0)]$.

{\bf Case (5)}. Flipping two 1-bits in $x_0$ (i.e., $|M_0(x'_0)| = |M_0(x_0)| +2$).
Using the reasoning similar to that for Case (2), we have that $L(x'_0) > L(x_0)$ and $x'_0$ cannot be accepted by RLS.

Summarizing the above analysis for the five cases gives that if $x'_0$ is accepted by RLS, then it satisfies at least one of two conditions: 
(1). $|M_{t(x'_0)}(x'_0)| \\ < |M_{0}(x_0)|$; 
(2). $|M_{t(x'_0)}(x'_0)| = |M_{0}(x_0)|$ and $cov[l_{t(x'_0)}(x'_0)] \le cov[l_{0}(x_0)]$ (recall that $t(x) = \arg_{t \in [0,1]} \max \{l_0'(x),l_1'(x)\}$).
That is, the gap between the numbers of jobs on the two machines cannot be enlarged during the optimization process.

The mutation considered in Case (1) can be generated by RLS with probability $1/4$ that decreases the gap between the numbers of jobs on the two machines by 2. As $||M_0(x_0)| - |M_1(x_0)|| \le n$, using the Additive Drift analysis~\citep{he2004study} gets that RLS takes expected runtime $O(n)$ to obtain the first equal solution $x_1$ with $|M_0(x_1)| = |M_1(x_1)|$ based on $x_0$.

{\bf Phase-2}. Obtaining the first balanced solution $x^*$ based on $x_1$. 

W.l.o.g., assume that $x_1$ is not balanced. 
To simplify the analysis, we define the \emph{potential} of the solution $x$ maintained during the process as 
$$p(x) = \sum_{i=1}^k |\alpha_i(x) - \beta_i(x)| = \sum_{i=1}^k |2 \alpha_i(x) - m|.$$
In the following, we show that the potential value cannot increase during the process. 
Note that once the first equal solution $x_1$ is obtained, then all solutions subsequently accepted by RLS are equal ones, thus only the mutations flipping a 0-bit and a 1-bit of $x_1$ are considered below. 
Assume that the considered mutation flips a 0-bit of $G_i$ and a 1-bit of $G_j$ in $x_1$, and denoted by $x'_1$ the solution obtained.
The potential change is 
\begin{eqnarray*}
\label{equ:RLS CCMSP-1 2}
\begin{aligned}
\Delta_p &=  p(x_1) - p(x'_1) \\
&= |2\alpha_i(x_1) - m| + |2\alpha_j(x_1) - m| - \left(|2\alpha_i(x'_1) - m| + |2\alpha_j(x'_1) - m| \right),
\end{aligned} 
\end{eqnarray*}
where $\alpha_i(x'_1) = \alpha_i(x_1) -1$ and $\alpha_j(x'_1) = \alpha_j(x_1) +1$.
The above discussion for Case (4) shows that $x'_1$ can be accepted by RLS if and only if $\Delta_{cov} = cov[l_0(x_1)] - cov[l_0(x'_1)] \ge 0$, where
\begin{eqnarray*}
\label{equ:RLS CCMSP-1 1}
\Delta_{cov}/2c &= \binom{\alpha_i(x_1)}{2} + \binom{\alpha_j(x_1)}{2} - \left(\binom{\alpha_i(x_1)-1}{2} + \binom{\alpha_j(x_1)+1}{2} \right)\\
&= \alpha_i(x_1)-1 - \alpha_j(x_1).
\end{eqnarray*}
We divide the analysis for the values of $\Delta_{p}$ and $\Delta_{cov}$ into four cases. 

{\bf Case (I)}. $\alpha_i(x_1) > \frac{m}{2} ~and~ \alpha_j(x_1) \ge \frac{m}{2}$.
Observe that $\Delta_{p} = 0$, but the value of $\Delta_{cov}$ depends on the relationship between $\alpha_i(x_1)$ and $\alpha_j(x_1)$.

{\bf Case (II)}. $\alpha_i(x_1) \le \frac{m}{2} \ ~and~ \ \alpha_j(x_1) \ge \frac{m}{2}$.
Observe that $\Delta_{p} = -4$, but $\Delta_{cov} < 0$, implying that $x'_1$ cannot be accepted by RLS.

{\bf Case (III)}. $\alpha_i(x_1) > \frac{m}{2} \ ~and~ \ \alpha_j(x_1) < \frac{m}{2}$.
Observe that $\Delta_{p} = 4$ and $\Delta_{cov} > 0$, implying that $x'_1$ can be accepted by RLS.

{\bf Case (IV)}. $\alpha_i(x_1) \le \frac{m}{2} \ ~and~ \ \alpha_j(x_1) < \frac{m}{2}$.
Observe that $\Delta_{p} = 0$, but the value of $\Delta_{cov}$ depends on the relationship between $\alpha_i(x_1)$ and $\alpha_j(x_1)$.

Summarizing the analysis for the above four cases gives that during the optimization process, the potential value cannot increase. 
As $x_1$ is not balanced, there exist $i,j \in [1,k]$ such that $\alpha_i(x_1) = |M_0(x_1) \cap G_i| > \frac{m}{2}$ and $\alpha_j(x_1) =|M_0(x_1) \cap G_j| < \frac{m}{2}$ (i.e., Case (III) holds), and the offspring obtained by the mutation flipping a 0-bit of $G_i$ and a 1-bit of $G_j$ in $x_1$ can be accepted by RLS.
Now we consider the probability to generate such a mutation. 
Let $S_{0}$ (resp., $S_1$) be a subset of $[1,k]$ such that for any $i \in S_0$ (resp., $i \in S_1$), $\alpha_i(x_1) > \beta_i(x_1)$ (resp., $\alpha_i(x_1) < \beta_i(x_1)$).
Note that $\alpha_i(x_1) > \beta_i(x_1)$ (resp., $\alpha_i(x_1) < \beta_i(x_1)$) indicates $\alpha_i(x_1) > \frac{m}{2}$ (resp., $\alpha_i(x_1) < \frac{m}{2}$).
Since $x_1$ is an equal solution, we have that 
\begin{eqnarray}
\sum_{i \in S_0} \alpha_i(x_1) - \beta_i(x_1) = \sum_{i \in S_1} \beta_i(x_1) - \alpha_i(x_1) = p(x_1)/2. 
\end{eqnarray}
Then combining Equality~(\ref{equ:RLS CCMSP-1 2}) with the facts that $\sum_{i \in S_0} \alpha_i(x_1) + \beta_i(x_1) = |S_0|m$ and $\sum_{i \in S_1} \alpha_i(x_1) + \beta_i(x_1) = |S_1|m$, we can get that 
$$\sum_{i \in S_0} \alpha_i(x_1) = \frac{p(x_1)}{4} + \frac{|S_0| m}{2}  \ge \frac{m}{2} + \frac{p(x_1)}{4} $$ and $$ \sum_{i \in S_1} \beta_i(x_1) = \frac{p(x_1)}{4} + \frac{|S_1| m}{2}  \ge \frac{m}{2}  + \frac{p(x_1)}{4}.$$
Thus there are at least $\frac{m}{2} + \frac{p(x_1)}{4}$ 0-bits in the groups $G_i$ with $i \in S_0$, and at least $\frac{m}{2}  + \frac{p(x_1)}{4}$ 1-bits in the groups $G_i$ with $i \in S_1$.
That is, RLS generates such a mutation with probability $\Omega((\frac{2m+p(x_1)}{4n})^2)$, and RLS takes expected runtime $O((\frac{n}{2m+p(x_1)})^2)$ to obtain an offspring $x'_1$ with $p(x'_1) = p(x_1)-4$.
Considering all possible values for the potential of the maintained solution (note that $1 \le p(x_1) \le n$ and $p(x_1)$ is a decreasing function) and defining $t:= p(x_1)$, the total expected runtime of RLS to obtain the optimal solution $x^*$ based on $x_1$ can be upper bounded by  
\begin{align*}
    \sum_{t = 1}^n O(\frac{n^2}{(t+2m)^2}) = &O(n^2) \sum_{t = 1}^n (t+2m)^{-2} \\
    =& O(n^2) \int_{1}^{n} (t+2m)^{-2} dt = O(n^2/m).
\end{align*}

Combining the expected runtime of Phase (1) (i.e., $O(n)$), and that of Phase (2) (i.e., $O(n^2/m) = O(kn)$), the total expected runtime of RLS to obtain an optimal solution to $I_1$ is $O(n^2/m) = O(kn)$.
\end{proof}

The theoretical performance of the (1+1) EA on the instance $I_1$ of CCMSP-1 is investigated in Theorem~\ref{thm:(1+1) EA CCMSP-1}, which runs in a similar way to that given in Theorem~\ref{thm:RLS CCMSP-1} for RLS, obtaining the equal solution first and then approaching the optimal solution to the instance. 

\begin{theorem}
\label{thm:(1+1) EA CCMSP-1}
The expected runtime of the (1+1) EA to obtain an optimal solution to the instance $I_1 = (a,c,d,\gamma,k,m)$ of CCMSP-1 is $O((k+m)n^2)$.
\end{theorem}

% \xiankun{reviewer 4 - Compared to Theorem 5, the proof of Theorem 6 traded a lot of effort for a lot of precision. Subsequently, in the experimental section, the difference between $O(nk)$ and $O(n^2(k+m))$ is claimed to be the basis of the performance difference between RLS and (1+1) EA, which is incorrect - as $O(...)$ is just an upper bound.
% I see no fundamental reason why (1+1) EA could be vastly slower than RLS, and, by looking at the proof for RLS, I conjecture that a simple application of drift analysis would make the same $O(n^2/m) = O(nk)$ bound for (1+1) EA.}

\begin{proof}
As the mutation of the (1+1) EA may flip more than two bits simultaneously, the reasoning given in Theorem~\ref{thm:RLS CCMSP-1} for the performance of RLS cannot be directly applied for the performance of the (1+1) EA, but the reasoning runs in a similar way.

We first consider the expected runtime of {\bf Phase (1)}: Obtaining the first equal solution $x_1$ based on the initial solution $x_0$ that is assumed to have $|M_0(x_{0})|>|M_1(x_{0})|$.
A vector function $v(x) = (|M_{t(x)}(x)|, b(x))$ is designed for the solutions $x$ maintained during Phase (1), where $b(x) = \sum_{i=1}^k \binom{|M_{t(x)}(x) \cap G_i|}{2}$.
For ease of notation, let $|M_{t(x)}(x)| = \ell$, where $\ell \in [\frac{n}{2}+1,n]$ (as $M_{t(x)}(x)$ is the fuller machine by Lemma~\ref{lem:inequal number}). Then $0 < b(x) \le \lfloor \frac{\ell}{m} \rfloor \binom{m}{2} + \binom{\ell \% m}{2} \le (\frac{\ell}{m} +1) \binom{m}{2}$, where the first $\le$ holds due to Lemma~\ref{lem:1}.
Hence the number of possible values of $v(x)$ can be upper bounded by 
$\sum_{\ell = \frac{n}{2}+1}^{n} (\frac{\ell}{m} +1) \binom{m}{2} = O(mn^2)$.
Observe that for any two solutions $x$ and $x'$, if $v(x) = v(x')$ then $L(x) = L(x')$. 
Thus the number of possible values of $L(x)$ during the process can be upper bounded by $O(mn^2)$ as well.

Consider a mutation flipping a $t(x)$-bit of $x$ (i.e., if $t(x) = 0$ then flipping a 0-bit; otherwise, a 1-bit).
By the discussion for Case (1) given in the proof of Theorem~\ref{thm:RLS CCMSP-1}, the solution $x'$ obtained by the mutation has $L(x') < L(x)$ and can be accepted by the (1+1) EA. The probability of the (1+1) EA to generate such a mutation is $1/2$. 
Thus combining the probability and the number of possible values of $L(x)$ gives that Phase (1) takes expected runtime $O(mn^2)$ to get the first equal solution $x_1$ based on $x_0$.

Now we consider the expected runtime of {\bf Phase (2)}: Obtaining an optimal solution based on $x_1$. 
As all solutions accepted subsequently are equal ones, we take $b(x)$ as the potential function, where the number of possible values of $b(x)$ can be bounded by $O(km^2)$.
By the reasoning given in Theorem~\ref{thm:RLS CCMSP-1}, a mutation flipping a 0-bit and a 1-bit that can obtain an improved solution can be generated by the (1+1) EA with probability $\Omega((\frac{m}{2n})^2)$.
Consequently, Phase (2) takes expected runtime $O(km^2 \cdot (\frac{2n}{m})^2) = O(kn^2)$ to obtain an optimal solution based on $x_1$.

In summary, the (1+1) EA  totally takes expected runtime $O((k+m)n^2)$ to obtain an optimal solution to $I_1$. 
\end{proof}

\section{Theoretical Performance for CCMSP-2}

The following observation shows that the discussion for CCMSP-2 is more complicated than that for CCMSP-1.
% \frank{is it really a discussion. replace "would be" by "is"?}-solved
\begin{observation}
\label{obs:CCMSP-2 not equal solution}
Given a solution $x$ to an instance $I_2 = (a,c,d,\gamma,k, \{m_i \mid i \in [1,k]\})$ of CCMSP-2, it remains uncertain  whether $l'_0(x) > l'_1(x)$ holds true even if $|M_0(x)| > |M_1(x)|$, when compared to CCMSP-1.
\end{observation}

% \xiankun{reviewer4-Lemma 7 is just an observation in fact, because one arguably cannot use it in a proof to prove something.}

\begin{proof}
We first show that the relationship between $cov[l_0(x)]$ and $cov[l_1(x)]$ remains ambiguous if only having $|M_0(x)| > |M_1(x)|$.  
There exists a variable $\delta_i(x)$ such that $\alpha_i(x) = m_i/2 + \delta_i(x)$ and $\beta_i(x) = m_i/2 - \delta_i(x)$ for any $i \in [1,k]$ (recall that the group $G_i$ in $I_2$ has size $m_i$).
Then we have $\sum_{i=1}^k \delta_i(x) > 0$ due to the fact $|M_0(x)| > |M_1(x)|$.
The gap between the covariances of the loads on the two machines is 
\begin{eqnarray}
\label{equ:CCMSP-2 not equal solution 2}
\begin{aligned}
cov[l_0(x)] - cov[l_1(x)] = 2c \sum_{i=1}^k \left(\binom{\alpha_i}{2} - \binom{\beta_i}{2}\right) &= 2c \sum_{i=1}^k (m_i -1)\delta_i .
\end{aligned}
\end{eqnarray}
Observe that Expression~(\ref{equ:CCMSP-2 not equal solution 2}) can be treated as a weighted version of $\sum_{i=1}^k \delta_i(x)$, and it is impossible to decide whether Expression~(\ref{equ:CCMSP-2 not equal solution 2}) is greater than 0.
Additionally, the relationship among the values of expected processing time $a$, variance $d$ and covariance $c$ are unrestricted, and we cannot derive that $l'_0(x) < l'_1(x)$ even if Expression~(\ref{equ:CCMSP-2 not equal solution 2}) is less than 0.
As a result, it is also impossible to determine whether $l'_0(x) > l'_1(x)$ holds.
\end{proof}

For ease of analysis, we consider two simplified variants of CCMSP-2, called CCMSP-$2^+$ and CCMSP-$2^-$, where CCMSP-$2^+$ and CCMSP-$2^-$ consider the extra constraints on the values of expected processing time $a$, variance $d$ and covariance $c$ given in Inequality~(\ref{equ:CCMSP-2^+}) and~(\ref{equ:CCMSP-2^-}), respectively.
\begin{eqnarray}
\label{equ:CCMSP-2^+}
\sqrt{\frac{(1-\gamma)}{\gamma} \left(nd + 2c\sum_{i=1}^{k} \binom{m_i}{2} \right)} < a
\end{eqnarray}
\begin{eqnarray}
\label{equ:CCMSP-2^-}
\begin{aligned}
    \sqrt{\frac{(1-\gamma)}{\gamma} 2c \sum_{i=1}^k \binom{m_i}{2}} - \sqrt{\frac{(1-\gamma)}{\gamma} 2c \left(\sum_{i=1}^k \binom{m_i}{2} - 1 \right)} > \\ na + \sqrt{\frac{(1-\gamma)}{\gamma} nd}
\end{aligned}
\end{eqnarray}

% \frank{makes me always skeptical}-solved
For any solution $x$ to any instance of CCMSP-$2^+$ and any $t \in [0,1]$, we note that $E[l_t(x)]$ contributes much more than $\sqrt{\frac{(1-\gamma)}{\gamma} var[l_t(x)]}$ to $l'_t(x)$ if $|M_t(x)| \ge 1$ (recall that $l'_t(x) = E[l_t(x)] + \sqrt{\frac{(1-\gamma)}{\gamma} var[l_t(x)]}$) under the extra constraint given in Inequality~(\ref{equ:CCMSP-2^+}), because 
\begin{align*}
    \sqrt{\frac{(1-\gamma)}{\gamma} var[l_t(x)]} \le \sqrt{\frac{(1-\gamma)}{\gamma}\left(nd + 2c\sum_{i=1}^{k} \binom{m_i}{2} \right)} &\\
     < a \le E[l_t(x)] = |M_t(x)|a.&
\end{align*}

Similarly, for any solution $x$ to any instance of CCMSP-$2^-$ and any $t \in [0,1]$, if $cov[l_t(x)] \ge 2c$ then the covariances of the jobs contribute much more than their expected processing time and variances to $l'_t(x)$ under the extra constraint give in Inequality~(\ref{equ:CCMSP-2^-}). 
More specifically, the covariances of the jobs contribute 

\begin{eqnarray}
\label{equ:contribution 1}
\begin{aligned}
\left(|M_{t}(x)| a  + \sqrt{\frac{(1-\gamma)}{\gamma} (|M_{t}(x)| d + cov[l_t(x)])} \right)& - \\
 \left(|M_{t}(x)| a + \sqrt{\frac{(1-\gamma)}{\gamma} |M_{t}(x)| d} \right)&
\end{aligned}
\end{eqnarray} 

to $l'_t(x)$, and their expected processing time and variances contribute 

\begin{eqnarray}
\label{equ:contribution 2}
\begin{aligned}
    \left(|M_{t}(x)| a + \sqrt{\frac{(1-\gamma)}{\gamma} (|M_{t}(x)| d + cov[l_t(x)])} \right) -& \\
    \left(\sqrt{\frac{(1-\gamma)}{\gamma} cov[l_t(x)]} \right).&
\end{aligned}
\end{eqnarray} 
Then  
\begin{eqnarray*}
\begin{aligned}
\label{equ:}
& \textrm{Expression}~(\ref{equ:contribution 1}) - \textrm{Expression}~(\ref{equ:contribution 2}) \\ 
=& \sqrt{\frac{(1-\gamma)}{\gamma} cov[l_t(x)]} - \left(|M_{t}(x)| a + \sqrt{\frac{(1-\gamma)}{\gamma} |M_{t}(x)| d}\right) \\
\ge& \sqrt{\frac{(1-\gamma)}{\gamma} cov[l_t(x)]} - \left(n a + \sqrt{\frac{(1-\gamma)}{\gamma} n d}\right) \\
>& \sqrt{\frac{(1-\gamma)}{\gamma} 2c} - \left( \sqrt{\frac{(1-\gamma)}{\gamma} 2c \sum_{i=1}^k \binom{m_i}{2}} - \sqrt{\frac{(1-\gamma)}{\gamma} 2c \left(\sum_{i=1}^k \binom{m_i}{2} - 1 \right)} \right) \\
>& 0.
\end{aligned}
\end{eqnarray*}

\subsection{NP-hardness of CCMSP-$2^+$ and CCMSP-$2^-$}

% \xiankun{reviewer 4 - Two-way Balanced Partition problem: the presented statement is not a proof that the two-way balanced partition problem is NP-hard.
% Anything about the existence of pseudo-polynomial algorithms for Partition and Two-way Balanced Partition has no relation to the reduction of Partition to Two-way Balanced Partition. So the NP-hardness proof should either be directly cited, or actually be performed. None of this is currently present in the paper.}

In the following, we first consider the NP-hardness of CCMSP-$2^+$.
\begin{lemma}
\label{lem:optimal solution to CCMSP-2^+}
Considering an instance $I_2^+ = (a,c,d,\gamma,k, \{m_i \mid i \in [1,k]\})$ of CCMSP-$2^+$, $L(x_1) < L(x_2)$ for any two solutions $x_1$ and $x_2$ to $I_2^+$ with $|M_{t(x_1)}(x_1)| < |M_{t(x_2)}(x_2)|$, and any optimal solution to $I_2^+$ is an equal one.
\end{lemma}

\begin{proof}
Because of the extra constraint of CCMSP-$2^+$ (i.e., Inequality~(\ref{equ:CCMSP-2^+})), for any solution $x$ to $I_2^+$, if $|M_0(x)| > |M_1(x)|$ (resp., $|M_0(x)| < |M_1(x)|$) then $l'_0(x) > l'_1(x)$ (resp., $l'_0(x) < l'_1(x)$).

Consider two solutions $x_1$ and $x_2$ to $I_2^+$ with $|M_0(x_1)| \ge |M_1(x_1)|$ and $|M_0(x_2)| \ge |M_1(x_2)|$ (implying that $L(x_1) = l'_0(x_1)$ and $L(x_2) = l'_0(x_2)$).
If $|M_0(x_1)| > |M_0(x_2)|$ then $l'_0(x_1) > l'_0(x_2)$ because of Inequality~(\ref{equ:CCMSP-2^+}). 
Combining the conclusion $l'_0(x_1) > l'_0(x_2)$ and that $L(x_1) = l'_0(x_1)$ and $L(x_2) = l'_0(x_2)$ gets $L(x_1) > L(x_2)$. Thus the first claim holds. 
By the first claim, it can get the second one.
\end{proof}

% \begin{observation}
% \label{}
% There exists an instance $\mathcal{I}$ of CCMSP-$2^+$ such that $m_i \ge 3$ is odd for all $1 \le i \le k$, and any optimal solution to $\mathcal{I}$ is a balanced solution.
% \end{observation}

Before showing the NP-hardness of CCMSP-$2^+$, it is necessary to introduce two related problems: \emph{Partition problem} and \emph{Two-way Balanced Partition problem}, where their formulations are given as follows.

\medskip
\noindent{\bf Partition problem}: Given a multiset $S$ that contains non-negative integers such that $\sum_{e \in S} e$ is even, can $S$ be partitioned into two subsets $S_1$ and $S_2$ such that 
$\sum_{e \in S_1} e = \sum_{e \in S_2} e$?
\medskip

\noindent{\bf Two-way Balanced Partition problem}: Given a multiset $S$ that contains non-negative integers such that both $|S|$ and $\sum_{e \in S} e$ are even, can $S$ be partitioned into two subsets $S_1$ and $S_2$ such that 
$$|S_1| = |S_2| \ \ \textrm{and} \ \ \sum_{e \in S_1} e = \sum_{e \in S_2} e?$$
\medskip

Note that the Partition problem is a well-known NP-hard problem~\citep{hayes2002computing}, for which there is a dynamic programming algorithm with runtime $O(|S| \cdot Sum)$, where $Sum = \sum_{e \in S} e$. 
The NP-hardness of the Partition problem is due to that $Sum$ cannot be upper bounded by a polynomial of the input size $|S| \cdot \log Max$, where $Max$ is the largest element in $S$.
The Two-way Balanced Partition problem can be proved to be NP-hard by reducing the Partition problem to it.
By extending the dynamic programming algorithm for the Partition problem (more specifically, designing a dynamic programming table with $|S| \cdot \frac{Sum}{2} \cdot \lceil \frac{|S|}{2} \rceil$), it is not hard to give an algorithm for the Two-way Balanced Partition problem with runtime $O(|S|^2 \cdot Sum)$.

\begin{lemma}
\label{obs:CCMSP-2^+ np-hard}
% CCMSP-$2^+$ is NP-hard.
Given an instance $I_2^+ = \{n,a,c,d,\gamma,k,\{m_i| 1 \le i \le k\}\}$ of CCMSP-$2^+$, if the number $n$ of jobs considered in $I_2^+$ is odd, then $I_2^+$ is polynomial-time solvable; otherwise, NP-hard.
\end{lemma}

% \xiankun{reviewer 4 - Following from the previous point, Lemma 9 is incorrect, because it reduces a problem, which is not yet proven to be NP-hard, to CCMSP-2+.}
\begin{proof}
For the computational hardness of CCMSP-$2^+$, the following discussion is to show that any instance of the Two-way Balanced Partition problem can be reduced to an instance of CCMSP-$2^+$ that has an even number of jobs, and any instance of CCMSP-$2^+$ that has an odd number of jobs can be solved in polynomial-time.
% \frank{has an}-solved

{\bf Case 1.} The instances of CCMSP-$2^+$ that consider odd number of  jobs.

Let $I_2^+ = (a,c,d,\gamma,k, \{m_i | i \in [1,k]\})$ be an arbitrary instance of CCMSP-$2^+$, where the number $n = \sum_{i=1}^k m_i$ of jobs considered is odd. 
Now we give a way to construct an optimal solution $x^*$ to $I_2^+$.
By Lemma~\ref{lem:optimal solution to CCMSP-2^+}, $x^*$ is an equal solution to $I_2^+$ that is assumed to have $|M_0(x^*)| =  |M_1(x^*)| +1 = \frac{n+1}{2}$. 
Because of the extra constraint of CCMSP-$2^+$ (i.e., Inequality~(\ref{equ:CCMSP-2^+})), $l'_0(x^*) > l'_1(x^*)$. 
Thus we only need to analyze the optimal allocation approach of the $\frac{n+1}{2}$ jobs on $M_0$ with respect to $x^*$ such that $cov[l_0(x^*)]$ is minimized.
By Lemma~\ref{lem:1}, $cov[l_0(x^*)] \ge 2ck \binom{\frac{n+1}{2k}}{2}$, where the low bound $2ck \binom{\frac{n+1}{2k}}{2}$ can be achieved by allocating $\frac{n+1}{2k}$ jobs to $M_0$ from each group. However, $\frac{n+1}{2k}$ may be not an integer, and $m_i$ may be less than $\frac{n+1}{2k}$ for some $i \in [1,k]$. 
Fortunately, by Lemma~\ref{lem:1}, it is not hard to derive that the optimal allocation approach of the $\frac{n+1}{2}$ jobs on $M_0$ with respect to $x^*$ such that $cov[l_0(x^*)]$ is minimized can be obtained as follows: For each $i \in [1,k]$ (the values of $\alpha_{i}(x^*)$ for all $i \in [1,k]$ are specified in order, and assume that the values of $\alpha_{j}(x^*)$ for all $j \in [1,i-1]$ have been specified), if 
$$m_i < (\frac{n+1}{2} - \sum_{j = 1}^{i-1} \alpha_{j}(x^*)) / (k+1-i),$$
then let $\alpha_{i}(x^*) = m_i$; otherwise, let $\alpha_{i}(x^*) = \lceil (\frac{n+1}{2} - \sum_{j = 1}^{i-1} \alpha_{j}(x^*)) / (k+1-i) \rceil $.
Observe that as $m_1 \le m_2 \le \ldots \le m_k$, once $\alpha_{i}(x^*)$ is set as $\lceil (\frac{n+1}{2} - \sum_{j = 1}^{i-1} \alpha_{j}(x^*)) / (k+1-i) \rceil$ for some $i \in [1,k]$, then for all $p \in [i+1,k]$, $0 \le \alpha_{i}(x^*) - \alpha_{p}(x^*) \le 1$.
In a word, an equal solution $x$ to $I^+_2$ is optimal if and only if $x$ satisfies the following property (assume that $|M_0(x)| > |M_1(x)|$).

\smallskip
\noindent {\bf Property-Odd}: For each $i \in [1,k]$, either $\alpha_{i}(x) = m_i$ or $0 \le \alpha_{max}(x) - \alpha_{i}(x) \le 1$, where $\alpha_{max}(x) = \max \{\alpha_{1}(x), \ldots, \alpha_{k}(x)\}$.
\smallskip

{\bf Case 2.} The instances of CCMSP-$2^+$ that consider even number of  jobs.

Now we give an instance $I_2^+ = (a,c,d,\gamma,k, \{m_i | i \in [1,k]\})$ of CCMSP-$2^+$, where $k$ is even and $m_i$ is odd for all $i \in [1,k]$, and show that to obtain an optimal solution to $I_2^+$, it is necessary to decide whether $I_2^+$ has a balanced solution $x'$ with $cov[l_0(x')] = cov[l_1(x')]$, as for any solution $x$ to $I_2^+$, by Lemma~\ref{lem:1} we have
\begin{align*}
    (cov[l_0(x)] + cov[l_1(x)])/2c &\ge (cov[l_0(x')] + cov[l_1(x')])/2c \\
    & = \sum_{i=1}^k \left(\binom{\lceil \frac{m_i}{2} \rceil}{2} + \binom{\lfloor \frac{m_i}{2} \rfloor}{2} \right).
\end{align*}

If there exists such a balanced solution $x'$, then it is optimal to $I_2^+$ due to Lemma~\ref{lem:optimal solution to CCMSP-2^+}.

Let $\alpha_i(x') = \frac{m_i}{2} + \delta_i$ and $\beta_i(x') = \frac{m_i}{2} - \delta_i$ for any $i \in [1,k]$, where $\delta_i \in \{\frac{1}{2}, -\frac{1}{2}\}$ (as $x'$ is balanced and each $m_i$ is odd). Then 
\begin{eqnarray*}
\begin{aligned}
(cov[l_0(x')] - cov[l_1(x')])/2c &= \sum_{i =1}^k \left(\binom{\frac{m_i}{2} + \delta_i}{2} - \binom{\frac{m_i}{2} - \delta_i}{2}\right) \\
&= \sum_{i=1}^k (m_i-1)\delta_i.
\end{aligned}
\end{eqnarray*}
The above discussion shows that to obtain an optimal solution to $I_2^+$, it is necessary to decide whether there exists a balanced solution $x'$ such that 
\begin{eqnarray}
\label{equ:2 CCMSP-2^+ np-hard}
\sum_{i =1}^k (m_i-1)\delta_i = 0.
\end{eqnarray}
That is, it is to decide whether there is a partition $\{S_0,S_1\}$ of $[1,k]$ such that $\sum_{i \in S_0} (m_i-1) = \sum_{j \in S_1} (m_j-1)$. Moreover, the balance of $x'$ implies that $|S_0| = |S_1|$.

Therefore, the problem can be transformed to decide whether the instance $S = \{m_i -1 \ | \ i \in [1,k]\}$ of the Two-way Balanced Partition problem is {\sc yes}, where $|S| = k$ and $\max \{m_i -1 \ | \ i \in [1,k]\} = m_k -1 \le n-1$. 
Note that the input size of $I_2^+$ is $O(k \log {m_k} + \log (a + c + d + \gamma))$, where 
$$O(k \log {m_k} + \log (a + c + d + \gamma)) = O(k \log {m_k} + \log n) = O(k \log {m_k})$$ 
if $c$, $d$ and $\gamma$ are small enough and $a = \theta(\sqrt{nd+n^2c})$ (due to the extra constraint of CCMSP-$2^+$).
%Under the situation, the runtime of the algorithm mentioned above to solve the instance $S = \{m_i -1 \ | \ i \in [1,k]\}$ of the Two-way Balanced Partition problem is exponential with the input size of $I_2^+$.

Now we show that any instance of the Two-way Balanced Partition problem can be polynomial-time reduced to an instance of CCMSP-$2^+$.
Consider an arbitrary instance $S = \{e_i | i \in [1,k]\}$ of the Two-way Balanced Partition problem. 
If some $e_i$ is odd, then we construct an instance $S' = \{2e_i | i \in [1,k]\}$, and note that $S$ is a yes-instance of the Two-way Balanced Partition problem if and only if $S'$ is a yes-instance of the Two-way Balanced Partition problem. 
Based on $S'$, we construct an instance 
$$I_{2'}^+ = (a,c,d, \gamma, k, \{m_i = 2e_i +1|i \in [1,k]\})$$ 
of CCMSP-$2^+$, where $c$, $d$ and $\gamma$ are small enough and $a = \theta(\sqrt{nd+n^2c})$ such that the extra constraint of CCMSP-$2^+$ is met.
By the discussion given above, it can decide whether or not $S$ is a yes-instance of the Two-way Balanced Partition problem, according to the makespan of the optimal solution obtained for $I_{2'}^+$.

%As mentioned before, there exists an algorithm with runtime $O(k^2 n)$ to solve the instance $S$.
%It is worthy to point out that the runtime $O(k^2 n)$ cannot be bounded by a polynomial of the input size $O(k \log {m_k} + \log (a + d + c))$ of $I_2^+$ (more specifically, if $O(k \log {m_k} + \log (a + d + c)) = O(k \log {m_k})$, then obviously $O(k^2 n)$ cannot be bounded by a polynomial of $O(k \log {m_k}$).
%Thus the partition $\{S_0,S_1\}$ of $Q$ cannot be obtained in polynomial-time, and the optimal solution $x^*$ can be constructed based on the partition $\{S_0,S_1\}$: For each $1 \le i \le k$, if $i \in S_0$, then $\alpha_i(x^*) = \lceil \frac{m_i}{2} \rceil$; if $i \in S_1$, then $\alpha_i(x^*) = \lfloor \frac{m_i}{2} \rfloor$; otherwise, $\alpha_i(x^*) = \frac{m_i}{2}$.

In summary, CCMSP-$2^+$ is NP-hard. 
\end{proof}

By Lemma~\ref{obs:CCMSP-2^+ np-hard}, we have the following corollary.

\begin{corollary}
\label{cor:CCMSP-2 np-hard}
CCMSP-2 is NP-hard. 
\end{corollary}

Now we consider the computational complexity of CCMSP-$2^-$.
Given an instance $I_2^-$ of CCMSP-$2^-$, to obtain its optimal solution, it is necessary to decide whether $I_2^-$ has a solution $x'$ such that 
$$ cov[l_0(x')] + cov[l_1(x')] = \sum_{i=1}^k \left( \binom{\lceil \frac{m_i}{2} \rceil}{2} + \binom{\lfloor \frac{m_i}{2} \rfloor}{2} \right) $$  and $cov[l_0(x')] = cov[l_1(x')].$
Then using the reasoning given in the proof of Lemma~\ref{obs:CCMSP-2^+ np-hard}, we get the NP-hardness of CCMSP-$2^-$.

\begin{lemma}
\label{lem:CCMSP-2^- np-hard}
CCMSP-$2^-$ is NP-hard.   
\end{lemma}

\subsection{Theoretical Performance for RLS on CCMSP-$2^+$}

For ease of notation, denoted by CCMSP-$2^+$-Even the variant of CCMSP-$2^+$ that considers an even number of jobs, and by CCMSP-$2^+$-Odd the variant of CCMSP-$2^+$ that considers an odd number of jobs.

\begin{theorem}
\label{thm:RLS for CCMSP-2^+ odd}
Given an instance $I_2^+ = (a,c,d,\gamma,k, \{m_i \mid i \in [1,k]\})$ of CCMSP-$2^+$-Odd, RLS takes expected runtime $O(\sqrt{k} n^3)$ to obtain an optimal solution to $I_2^+$, where $n = \sum_{i=1}^k m_i$.
\end{theorem}

\begin{proof}
Let $x_0$ be the initial solution maintained by RLS that is assumed to be not an equal solution. 
The optimization process of RLS based on $x_0$ discussed below is divided into two phases. 

{\bf Phase-1}. Obtaining the first equal solution $x_1$ based on $x_0$.

Let the potential of the solution $x$ maintained during Phase-1 be $p_1(x) = ||M_0(x)| - |M_1(x)||$.
Observe that $0 \le p_1(x) \le n$.
Moreover, by Lemma~\ref{lem:optimal solution to CCMSP-2^+}, we have that for any two solutions $x'$ and $x''$ to $I_2^+$, if $p_1(x') < p_1(x'')$ then $L(x') < L(x'')$.

The mutation of RLS flipping exactly a $t(x)$-bit of $x$ (i.e., if $t(x) = 0$ then flipping a 0-bit; otherwise, a 1-bit), can be generated by RLS with probability $1/4$, and the obtained solution $x'$ has a smaller potential value compared to $x$, more specifically, $p_1(x') = p_1(x) -2$. Combining $p_1(x') = p_1(x) -2$ with the conclusion given above gets $L(x') < L(x)$, and $x'$ can be accepted by RLS.  
Then using the Additive Drift analysis~\citep{he2004study}, based on the upper bound $n$ of $p_1(x_0)$, we can derive that Phase-1 takes expected runtime $O(n)$ to obtain the first equal solution $x_1$, based on $x_0$.
Moreover, note that after the acceptance of $x_1$, any non-equal solution cannot be accepted anymore by Lemma~\ref{lem:optimal solution to CCMSP-2^+}. W.l.o.g., assume that $|M_0(x_1)| = |M_1(x_1)| + 1 = \frac{n+1}{2}$.

{\bf Phase-2}. Obtaining the first optimal solution based on $x_1$. 

{\bf Case (1)}. $cov[l_0(x_1)] < cov[l_1(x_1)]$. 

Consider a mutation flipping exactly one bit of $x_1$, and denote by $x'_1$ the obtained solution. 
If the 0-bit is fipped, then $cov[l_1(x'_1)] > cov[l_1(x_1)] > cov[l_0(x_1)]$ and $E[l_1(x'_1)] = E[l_0(x_1)] = E[l_0(x'_1)] +a = E[l_1(x_1)] + a$. 
Thus $L(x'_1) = l'_1(x'_1) > L(x_1) = l'_0(x_1)$, and $x'_1$ cannot be accepted by RLS.
Besides, once the  1-bit is flipped, $x'_1$ is not an equal solution and cannot be accepted. 
Therefore, for Case (1), any mutation flipping exactly one bit of $x_1$ cannot get an improved solution.

Consider a mutation flipping exactly two bits of $x_1$, and denote by $x'_1$ the obtained solution.
If the two flipped bits are both 1 or both 0, then the obtained solution $x'_1$ is not an equal solution and cannot be accepted by RLS.
Thus the following discussion only considers the mutation flipping a 0-bit of $G_i$ and a 1-bit of $G_j$ in $x_1$.
That is,  $|M_0(x'_1)| = |M_0(x_1)| = |M_1(x'_1)| + 1 = |M_1(x_1)| + 1$. 
To decide whether the obtained solution $x'_1$ is improved, we only need to consider the covariances of $l_0(x_1)$ and $l_0(x'_1)$, where
\begin{eqnarray*}
\begin{aligned}
& cov[l_0(x_1)] - cov[l_0(x'_1)] \\
=& 2c \left( \binom{\alpha_i(x_1)}{2} + \binom{\alpha_j(x_1)}{2} - \binom{\alpha_i(x_1) - 1}{2} - \binom{\alpha_j(x_1) +1}{2}\right) \\ 
=& 2c(\alpha_i(x_1) - 1 - \alpha_j(x_1)).
\end{aligned}
\end{eqnarray*}
If $\alpha_i(x_1) - \alpha_j(x_1) \ge 1$ then $cov[l_0(x'_1)] \le cov[l_0(x_1)]$, and $x'_1$ can be accepted by RLS.

Assume that RLS obtains a solution $x^*_1$ from $x_1$, on which all possible mutations flipping exactly one 0-bit and one 1-bit of $x^*_1$ cannot get an improved solution, where the 0-bit and 1-bit are of $G_i$ and $G_j$, respectively.
Note that $|M_0(x^*_1)| = |M_1(x^*_1)| + 1 = \frac{n+1}{2}$ based on the above discussion. 
Then $x^*_1$ satisfies the property: For any $i,j \in [1,k]$ with $i \neq j$, if $\alpha_i(x^*_1) - \alpha_j(x^*_1) \ge 2$, then all jobs of $G_j$ are allocated to $M_0$ with respect to $x^*_1$, i.e., $|M_0(x^*_1) \cap G_j| = \alpha_j(x^*_1) = m_j$ and $\beta_j(x^*_1) = 0$.
In other words, for any $j \in [1,k]$, either $0 \le \alpha_{max}(x^*_1) - \alpha_j(x^*_1) \le 1$ or $\alpha_j(x^*_1) = m_j$, where $\alpha_{max}(x^*_1) = \max \{\alpha_1(x^*_1), \ldots, \alpha_k(x^*_1)\}$. 
Thus $x^*_1$ satisfies Property-Odd given in the proof of Lemma~\ref{obs:CCMSP-2^+ np-hard}, and $x^*_1$ is an optimal solution to $I_2^*$.

Now we consider the expected runtime of RLS to obtain $x^*_1$ based on $x_1$.
Let the potential of the solution $x$ maintained during the phase be $p_{21}(x) = cov[l_0(x)]/2c$.
Note again that the above discussion shows $|M_0(x)| = |M_1(x)| + 1$.
For $i_{max} = \arg \max \{\alpha_1(x), \ldots, \alpha_k(x)\}$, we have $\binom{\alpha_{i_{max}}(x)}{2} \ge \frac{p_{21}(x)}{k}$, implying that 
$$\alpha_{i_{max}}(x) \ge \frac{\sqrt{1+ \frac{8 p_{21}(x)}{k}}+1}{2}.$$
Since $x$ does not satisfy Property-Odd, there exists a $j' \in [1, k]$ with $j' \neq i_{max}$ such that $\alpha_{i_{max}}(x) - \alpha_{j'}(x) \ge 2$ but $\alpha_{j'}(x) < m_{j'}$.
Thus $\beta_{j'}(x) \ge 1$.
The mutation flipping a 0-bit of $G_{i_{max}}$ and a 1-bit of $G_{j'}$ can be generated by RLS with probability 
$$\Omega(\frac{\alpha_{i_{max}}(x) \cdot \beta_{j'}(x)}{n^2}) = \Omega(\frac{\alpha_{i_{max}}(x)}{n^2}) = \Omega(\frac{1}{n^2} \sqrt{\frac{p_{21}(x)}{k}}),$$ 
and the potential value of the resulting solution is decreased by at least 1 compared to $x$.
Observe that the upper bound of $p_{21}(x_1)$ and lower bound of $p_{21}(x^*_1)$ are $\binom{\frac{n+1}{2}}{2}$ and $k \binom{\frac{n+1}{2k}}{2}$, respectively.
Considering all possible potential values of $x$, we have that the expected runtime of the phase to obtain $x^*_1$ based on $x_1$ can be bounded by 
\begin{align*}
    \sum_{t = k \binom{\frac{n+1}{2k}}{2}}^{\binom{\frac{n+1}{2}}{2}} O(\frac{\sqrt{k} n^2}{\sqrt{t}}) 
= O(\sqrt{k} n^2) \sum_{t = k \binom{\frac{n+1}{2k}}{2}}^{\binom{\frac{n+1}{2}}{2}} t^{-\frac{1}{2}} 
=& O(\sqrt{k} n^2) \int_{k \binom{\frac{n+1}{2k}}{2}}^{\binom{\frac{n+1}{2}}{2}} t^{-\frac{1}{2}} dt \\
=& O(\sqrt{k}n^3).
\end{align*}

{\bf Case (2)}. $cov[l_0(x_1)] \ge cov[l_1(x_1)]$. 

The discussion for Case (2) is similar to that for Case (1), and the difference is that the mutation flipping one bit may generate an improved solution. Thus the fuller machine may be $M_0$ or $M_1$. 
However, no matter which one is the fuller machine, the value $cov[l_{t(x)}(x)]$ cannot increase during Phase-2. 
Thus we let the potential of the solution $x$ maintained during the phase be $p_{22}(x) = cov[l_{t(x)}(x)]/2c$.

By the reasoning given for Case (1), for a mutation flipping exactly a 0-bit of $G_i$ and a 1-bit of $G_j$ in $x$, if $t(x) = 0$ and $\alpha_i(x) - \alpha_j(x) > 1$, or $t(x) = 1$ and $\beta_j(x) - \beta_i(x) > 1$, then $cov[l_{t(x')}(x')] < cov[l_{t(x)}(x)]$ for the obtained solution $x'$ (note that $t(x') = t(x)$ as the number of jobs on the two machines do not change with respect to the mutation). Thus $x'$ is improved and can be accepted by RLS. Afterwards, using the reasoning similar to that given for Case (1), we can get that Phase (2) takes expected runtime $O(\sqrt{k} n^3)$ to obtain an optimal solution to $I_2^+$ under Case (2).

Combining the expected runtime of Phase (1) (i.e., $O(n)$), and that of Phase (2) (i.e., $O(\sqrt{k} n^3)$), the total expected runtime of RLS to obtain an optimal solution to $I_2^+$ is $O(\sqrt{k} n^3)$.
\end{proof}

Now we analyze the properties of the local optimal solution obtained by RLS for the instance of CCMSP-$2^+$-Even, and the expected runtime of RLS to get it.  

\begin{theorem}
\label{thm:RLS for CCMSP-2^+ even}
Considering an instance $I_2^+ = (a,c,d,\gamma,k, \{m_i | i \in [1,k]\})$ of CCMSP-$2^+$-Even, RLS takes expected runtime $O(n^4)$ to obtain a local optimal solution $x^*$ that is an equal solution such that either $|cov[l_0(x^*)] - cov[l_1(x^*)]| \le 2c(\beta_i(x^*) - \beta_j(x^*) + 1)$ for any $i,j \in [1,k]$ with $\alpha_i(x^*) > \alpha_j(x^*) +1$ or $cov[l_{t(x^*)}(x^*)] \le \frac{c}{4} (\frac{n^2}{k} - 2n + k)$, where $n = \sum_{i=1}^k m_i$. Additionally, the approximation ratio of the solution $x^*$ is $1+\frac{2}{n}$.
\end{theorem}

\begin{proof}
Let $x_0$ be the initial solution maintained by RLS that is assumed to be not an equal solution. 
The proof runs in a similar way to that of Theorem~\ref{thm:RLS for CCMSP-2^+ odd}, dividing the optimization process into two phases: {\bf Phase-1}, obtaining the first equal solution $x_1$ from $x_0$; {\bf Phase-2}, optimizing the solution $x_1$.
Moreover, the analysis for Phase-1 is the same as that given in the proof of Theorem~\ref{thm:RLS for CCMSP-2^+ odd}, i.e., Phase-1 takes expected runtime $O(n)$. 

Now we consider Phase-2.
W.l.o.g., assume that $cov[l_0(x_1)] > cov[l_1(x_1)]$.
Let $\Delta(x_1) = cov[l_0(x_1)] - cov[l_1(x_1)]$.
Observe that any mutation flipping exactly one bit or two 0-bits or two 1-bits of $x_1$ cannot get an improved solution. Thus the following discussion only considers the mutation flipping exactly a 0-bit of $G_i$ and a 1-bit of $G_j$ in $x_1$. Denote by $x'_1$ the obtained solution. We have 
\begin{eqnarray*}
\begin{aligned}
&cov[l_0(x'_1)] \\
=& cov[l_0(x_1)] - \\
& \qquad 2c \left[ \left(\binom{\alpha_i(x_1)}{2} + \binom{\alpha_j(x_1)}{2} \right) - \left(\binom{\alpha_i(x_1)-1}{2} + \binom{\alpha_j(x_1)+1}{2} \right) \right] \\
=& cov[l_0(x_1)] + 2c (\alpha_j(x_1) - \alpha_i(x_1)+1)
\end{aligned}
\end{eqnarray*}
and
\begin{eqnarray*}
\begin{aligned}
&cov[l_1(x'_1)] \\ 
=& cov[l_1(x_1)] - \\
& \qquad 2c \left[ \left(\binom{\beta_i(x_1)}{2} + \binom{\beta_j(x_1)}{2} \right) - \left(\binom{\beta_i(x_1)+1}{2} + \binom{\beta_j(x_1)-1}{2} \right) \right] \\
=& cov[l_1(x_1)] + 2c (\beta_i(x_1) - \beta_j(x_1) +1).
\end{aligned}
\end{eqnarray*}
If $\alpha_j(x_1) \le \alpha_i(x_1) -1$ (i.e., $cov[l_0(x'_1)] \le cov[l_0(x_1)]$) and $\beta_i(x_1) - \beta_j(x_1) +1 \le \Delta(x_1)/2c$ (i.e., $cov[l_1(x'_1)] \le cov[l_0(x_1)]$), then $L(x'_1) \le L(x_1)$ and $x'_1$ can be accepted by RLS, and
\begin{eqnarray*}
\begin{aligned}
cov[l_0(x'_1)] - cov[l_1(x'_1)] 
&= \Delta(x_1) + 2c \left[ \left(\alpha_j(x_1) + \beta_j(x_1)) - (\alpha_i(x_1) + \beta_i(x_1) \right) \right] \\
&= \Delta(x_1) + 2c(m_j - m_i).
\end{aligned}
\end{eqnarray*}

Now we assume that RLS obtains a solution $x^*_1$ such that any mutation flipping a 0-bit and a 1-bit of $x^*_1$ cannot get an improved solution, and $cov[l_0(x^*_1)] \ge cov[l_1(x^*_1)]$. 
Let $i$ be an arbitrary integer in $[1,k]$.  
The above discussion shows that for any $j \in [1,k]$, if $\alpha_j(x^*_1) < \alpha_{i}(x^*_1) - 1$ then 
$$\beta_{i}(x^*_1) - \beta_j(x^*_1) +1 \ge \Delta(x^*_1)/2c.$$

If there is no $j$ with $\alpha_j(x^*_1) < \alpha_{i}(x^*_1) - 1$, then for each $j \in [1,k]$, $0 \le \alpha_{{max}}(x^*_1) - \alpha_j(x^*_1) \le 1$, where $\alpha_{max}(x^*_1) = \max\{\alpha_1(x^*_1), \alpha_2(x^*_1), \ldots, \alpha_k(x^*_1)\}$.
Now we bound the value of $cov[l_0(x^*_1)]$. Let $\tau = |\{j \in [1,k]| \alpha_j(x^*_1) = \alpha_{{max}}(x^*_1) - 1\}|$ with $0 \le \tau \le k-1$.
Then $$(k- \tau)\alpha_{{max}}(x^*_1) + \tau(\alpha_{{max}}(x^*_1) -1) = n/2$$
indicates that $\alpha_{{max}}(x^*_1) = \frac{n}{2k} + \frac{\tau}{k}$, and 
\begin{eqnarray*}
\begin{aligned}
cov[l_0(x^*_1)]/2c &= (k- \tau)\binom{\alpha_{{max}}(x^*_1)}{2} + \tau \binom{\alpha_{{max}}(x^*_1)-1}{2} \\
&= k \binom{\alpha_{{max}}(x^*_1)}{2} - \tau \left(\binom{\alpha_{{max}}(x^*_1)}{2} - \binom{\alpha_{{max}}(x^*_1)-1}{2} \right) \\
&= (\frac{n^2}{8k} + \frac{\tau^2}{2k} + \frac{n\tau}{2k} - \frac{n}{4} - \frac{\tau}{2}) - (\frac{n\tau}{2k} + \frac{\tau^2}{k} - \tau)\\
&= \frac{n^2}{8k} - \frac{n}{4} - (\frac{\tau^2}{2k} - \frac{\tau}{2}) \\
& \le \frac{n^2}{8k} - \frac{n}{4} + \frac{k}{8},
\end{aligned}
\end{eqnarray*} 
where the term $\frac{\tau^2}{2k} - \frac{\tau}{2}$ gets its minimum value $-\frac{k}{8}$ when $\tau = \frac{k}{2}$.

Now we analyze the expected runtime of RLS to get $x^*_1$ from $x_1$.
Let the potential function be $p(x) = cov[l_{t(x)}(x)]/2c$, where $x$ is a solution maintained by RLS during the process.
Observe that $p(x)$ cannot increase during the process. 
The probability of RLS to generate such a mutation mentioned above is $\Omega(1/n^2)$, and the potential value decreases by at least 1.
As $p(x_1) - p(x^*_1)$ can be upper bounded by $O(n^2)$, using the Additive Drift analysis~\citep{he2004study} implies RLS takes expected runtime $O(n^4)$ to obtain $x^*_1$ from $x_1$.

In summary, RLS in totoal takes expected runtime $O(n^4)$ to obtain an equal solution $x^*_1$ such that either $|cov[l_0(x^*_1)] - cov[l_1(x^*_1)]| \le 2c(\beta_i(x^*_1) - \beta_j(x^*_1) + 1)$ for any $i,j \in [1,k]$ with $\alpha_i(x^*_1) > \alpha_j(x^*_1) +1$ or $cov[l_{t(x^*_1)}(x^*_1)] \le \frac{c}{4} (\frac{n^2}{k} - 2n + k)$.
For the obtained solution $x^*_1$, we have that $L(x^*_1) \le (\frac{n}{2}+1)a$ and $L(x_{opt}) \ge \frac{n}{2}a$, where $x_{opt}$ is an optimal solution to $I^+_2$. Thus the approximation ratio of $x^*_1$ is $\frac{L(x^*_1)}{L(x_{opt})} \le 1 + \frac{2}{n}$. 
\end{proof}

\subsection{Theoretical Performance for RLS on CCMSP-$2^-$}

As the discussion for the approximation ratio of the local optimal solutions obtained by RLS for the instances of CCMSP-$2^-$ would be very complicated, we first analyze the case that the instances of CCMSP-$2^-$ satisfying some specific condition in the following lemma.

\begin{lemma}
\label{lem:RLS for CCMSP-2^- small instance}
Given an instance $I_2^- = (a,c,d,\gamma,k, \{m_i \mid i \in [1,k]\})$ of CCMSP-$2^-$ such that $[2\sum_{i \in S} (m_i-2)(m_i+1)] - 4 m_k -8 \le 0$, where $S = \{i \in [1,k]| m_i \ge 2\}$, the approximation ratio of the local optimal solution obtained by RLS is at most $\sqrt{3}$.
\end{lemma}
\begin{proof}
%Observe that if $m_i = 2$ for all $i \in [1,k]$ then $I_2^-$ is also an instance of CCMSP-1, for which RLS takes expected runtime $O(kn)$ to obtain an optimal solution by Theorem~\ref{thm:RLS CCMSP-1}. 
%In the following discussion, we assume that there exists at least one group with size greater than 2. 
To meet the condition of $[2\sum_{i \in S} (m_i-2)(m_i+1)] - 4 m_k -8 \le 0$, $I_2^-$ cannot have a group with size at least 5, and cannot have two groups with size 4 simultaneously.  
Moreover, if $I_2^-$ has a group with size 4 , then it has no group with 3; if $I_2^-$ has no group with size at least 4, then it has at most two groups with size 3. 
Now we analyze a local optimal solution $x$ maintained by RLS for $I_2^-$. 
Observe that there are no two groups $G_i$ and $G_j$ with $i,j \in S$ such that $\alpha_i(x) = m_i$ and $\beta_j(x) = m_j$; otherwise, flipping a 0-bit of $G_i$ and 1-bit of $G_j$ in $x$ obtains an improved solution. 
Thus if there is a group $G_i$ with $i \in S$ such that $\alpha_i(x) = m_i$ (resp., $\beta_i(x) = m_i$), then there is no group $G_j$ with $j \in S$ such that $\beta_j(x) = m_j$ (resp., $\alpha_j(x) = m_j$).

Case (1). $I_2^-$ contains a group with size 4, i.e., $m_k = 4$ and $m_{i} = 2$ for all $i \in S \setminus \{k\}$.

%It is easy to see that if all the $n$ many jobs are allocated to the same machine, then the corresponding solution is not a local optimal one. Thus for the local optimal solution $x$, we have that there is no $i \in [1,k]$ such that $|M_0(x) \cap G_i| = m_i$ or $|M_1(x) \cap G_i| = m_i$.

Case (1.1). $|cov[l_0(x)] - cov[l_1(x)]|/2c \ge 2$. W.l.o.g., assume that $cov[l_0(x)] \\ > cov[l_1(x)]$. 

Observe that there is no $i \in S$ such that $\alpha_i(x) = m_i$; otherwise, contradicting the fact that $x$ is local optimal. 
Similarly, $\alpha_k(x) \neq 3$; otherwise, flipping a 0-bit of $G_k$ gets an improved solution.
Thus $\alpha_i(x) \le \frac{m_i}{2}$ for each $i \in S$, implying that Case (1.1) does not exist.

Case (1.2). $|cov[l_0(x)] - cov[l_1(x)]|/2c = 1$. W.l.o.g., assume that $cov[l_0(x)] \\ > cov[l_1(x)]$. 

Observe that there is no $i \in S$ such that $\alpha_i(x) = m_i$.
If $\alpha_k(x) = 3$, then there is at least one group $G_i$ such that $\beta_i(x) = 2$; otherwise, the condition of Case (1.2) cannot hold. 
Afterwards, flipping a 0-bit of $G_k$ and a 1-bit of $G_i$ in $x$ obtains an improved solution, contradicting the fact that $x$ is local optimal. 
If $\alpha_k(x) \le 2$, then there is at least one group $G_i$ with $i \in S \setminus \{k\}$ such that $\alpha_i(x) = 2$ to meet the condition of Case (1.2), contradicting the observation given above. 
Summarizing the above discussion gives that Case (1.2) does not exist.

Case (1.3). $cov[l_0(x)] = cov[l_1(x)]$.
To meet the condition of Case (1.3), if there is a group $G_i$ such that $\alpha_i(x) = m_i$, then there is a group $G_j$ such that $\beta_j(x) = m_j$, contradicting the observation given above. 
Then using the reasoning similar to that given for Case (1.2), we can derive that $cov[l_0(x)] = cov[l_1(x)] = 2c$ (specifically, $\alpha_i(x) = \frac{m_i}{2}$ for each $i \in S$).

Case (2). $I_2^-$ contains exactly two groups with size 3 (i.e., $m_k = m_{k-1} = 3$).

Case (2.1). $|cov[l_0(x)] - cov[l_1(x)]|/2c \ge 2$. W.l.o.g., assume that $cov[l_0(x)] \\ > cov[l_1(x)]$. 

Observe that there is no $i \in S$ such that $\alpha_i(x) = m_i$, and $\alpha_k(x) \neq 2$ and $\alpha_{k-1}(x) \neq 2$; otherwise, contradicting the fact that $x$ is local optimal. 
That is, $\alpha_i(x) \le \frac{m_i}{2}$ for each $i \in S$, implying that the condition of Case (2.1) does not hold, i.e., Case (2.1) does not exist.

Case (2.2). $|cov[l_0(x)] - cov[l_1(x)]|/2c = 1$. W.l.o.g., assume that $cov[l_0(x)] \\ > cov[l_1(x)]$. 

Observe that there is no $i \in S$ such that $\alpha_i(x) = m_i$.
If $\alpha_k(x) = \alpha_{k-1}(x) = 2$, then there is at least one group $G_i$ with $i \in S \setminus \{k-1,k\}$ such that $\beta_i(x) = 2$; otherwise, the condition of Case (2.2) cannot hold. 
Flipping a 0-bit of $G_k$ and a 1-bit of $G_i$ in $x$ obtains an improved solution, contradicting the fact that $x$ is local optimal. 
If $\alpha_k(x) = 2$ and $\alpha_{k-1}(x) \le 1$, then there is a group $G_i$ with $i \in S \setminus \{k-1,k\}$ such that $\alpha_i(x) = 2$.
Flipping a 0-bit of $G_i$ gets an improved solution, contradicting the fact that $x$ is local optimal. 
The same reasoning applies to the cases that $\alpha_k(x) = 1$ and $\alpha_{k-1}(x) \le 2$, and that $\alpha_k(x) = 0$ and $\alpha_{k-1}(x) \le 2$.
Summarizing the above discussion gives that Case (2.2) does not exist.

Case (2.3). $cov[l_0(x)] = cov[l_1(x)]$. 
Observe that if $\alpha_k(x) = 3$ (resp., $\beta_k(x) = 3$), then $\beta_{k-1}(x) = 3$ or there is some $i \in S \setminus \{k\}$ such that $\beta_i(x) = m_i = 2$ (resp., $\alpha_{k-1}(x) = 3$ or there is some $i \in S \setminus \{k\}$ such that $\alpha_i(x) = m_i = 2$), contradicting the fact that $x$ is local optimal. The same reasoning applies to $\alpha_{k-1}(x)$ and $\beta_{k-1}(x)$.
If $\alpha_k(x) = \alpha_{k-1}(x) = 2$, then there are two groups $G_i$ and $G_j$ with $i,j \in S \setminus \{k-1,k\}$ such that $\beta_i(x) = \beta_j(x) = 2$; otherwise, the condition of Case (2.3) cannot hold. 
Moreover, for any other group $G_p$ with $p \in S \setminus \{i,j,k-1,k\}$, $\alpha_p(x) = \beta_p(x) = 1$.
That is, if $\alpha_k(x) = \alpha_{k-1}(x) = 2$, then $cov[l_0(x)] = cov[l_1(x)] = 4c$.
If $\alpha_k(x) = 2$ and $\alpha_{k-1}(x) = 1$, combining the observation given above we can derive that $cov[l_0(x)] = cov[l_1(x)] =2c$ (specifically, for any group $G_i$ with $i \in S \setminus \{k-1,k\}$, $\alpha_i(x) = \beta_i(x) = 1$).
The same reasoning applies to the cases that $\alpha_k(x) = 1$ and $\alpha_{k-1}(x) = 2$, and that $\alpha_k(x) = 1 = \alpha_{k-1}(x)$.

Case (3). $I_2^-$ contains exactly one group with size 3 (i.e., $m_k = 3$ and $m_{i} = 2$ for each $i \in S \setminus \{k\}$). 
Using the reasoning similar to that given for Case (1.1), we get that Case (3.1). $|cov[l_0(x)] - cov[l_1(x)]|/2c \ge 2$ does not exist.

Case (3.2). $|cov[l_0(x)] - cov[l_1(x)]|/2c = 1$. W.l.o.g., assume that $cov[l_0(x)] \\ > cov[l_1(x)]$. 
Observe that $\alpha_i(x)$ cannot be $m_i$ for each $i \in S$ and $\alpha_k(x) = 2$; otherwise, $x$ is not local optimal.
Consequently, we can derive that $\alpha_i(x) = \beta_i(x) = 1$ for each $i \in S \setminus \{k\}$.

Case (3.3). $cov[l_0(x)] = cov[l_1(x)]$. 
If $\alpha_k(x) = 3$ (resp., $\beta_k(x) = 3$), then there is a group $G_j$ with $j \in S \setminus \{k\}$ such that $\beta_j(x) = 2$ (resp., $\alpha_j(x) = 2$), contradicting the observation given above.
That is, $\alpha_k(x) \le 2$ and $\beta_k(x) \le 2$.
If $cov[l_0(x)] = cov[l_1(x)] \ge 4c$, then there are groups $G_i$ and $G_j$ with size 2 such that $\alpha_i(x) = 2 = \beta_j(x)$, contradicting the observation given above.
If $cov[l_0(x)] = cov[l_1(x)] = 2c$, then w.l.o.g., assume that $\alpha_k(x) = 2$ and $\beta_k(x) = 1$. 
Thus there is a group $G_i$ with $i \in S \setminus \{k\}$ such that $\beta_i(x) = 2$, and for any other group $G_j$ with $j \in S \setminus \{i,k\}$, $\alpha_j(x) = \beta_j(x) = 1$. Note that $|M_1(x)|- |M_0(x)| = 1$ and $t(x) = 1$. 

Case (4). $I_2^-$ contains no group with size greater than 2. 
Using the reasoning similar to that given above, we have that $cov[l_0(x)] = cov[l_1(x)] =0$ (i.e., $\alpha_i(x) = \beta_i(x) = 1$ for each $i \in S$).

Summarizing the above discussion gives that for Case (1), $cov[l_0(x)] = cov[l_1(x)] =2c$; for Case (2), $cov[l_0(x)] = cov[l_1(x)] =4c$ or $cov[l_0(x)] = cov[l_1(x)] =2c$; for Case (3), $cov[l_0(x)] = cov[l_1(x)] =2c$ or $cov[l_0(x)] + cov[l_1(x)] =2c$;  and for Case (4), $cov[l_0(x)] = cov[l_1(x)] =0$. 
Now we analyze the allocation of the groups with size 1. 
Obviously, for Case (1) and (4), as $cov[l_0(x)] = cov[l_1(x)]$ and $\sum_{i \in S} m_i$ is even, the groups with size 1 are allocated evenly to the two machines; otherwise, an improvement can be achieved by flipping one bit or two bits.
Therefore, $x$ is a balanced solution that is globally optimal to $I_2^-$.
The discussion for Case (2) is divided into two subcases given below. 

Case (I). $cov[l_0(x)] = cov[l_1(x)] =4c$. 
By the above discussion for Case (2), w.l.o.g., we assume that $\alpha_k(x) = \alpha_{k-1}(x) = 2$.
Thus there are two groups $G_i$ and $G_j$ with size 2 such that $\beta_i(x) = \beta_j(x) = 2$. 
If there is no group with size 1, then for the approximation ratio compared with the optimal solution $x_{opt}$, 
\begin{eqnarray*}
\begin{aligned}
\frac{L(x)}{L(x_{opt})} \le& \frac{(|S|+2)a+ \sqrt{\frac{1-\gamma}{\gamma}(4c+(|S|+2)d)}}{(|S|+1)a+ \sqrt{\frac{1-\gamma}{\gamma}(2c+(|S|+1)d)}} \\
\le& \frac{(|S|+2)a+ \sqrt{\frac{1-\gamma}{\gamma}4c} + \sqrt{\frac{1-\gamma}{\gamma}(|S|+2)d}}{(|S|+1)a+ \sqrt{\frac{1-\gamma}{\gamma}(2c+(|S|+1)d)}} \le \frac{\sqrt{\frac{1-\gamma}{\gamma}6c}}{\sqrt{\frac{1-\gamma}{\gamma}2c}} \le \sqrt{3}.
\end{aligned}
\end{eqnarray*}
If there are some groups with size 1, then they are allocated to balance the number of jobs on the two machines. 
Similarly, we can also get that $\frac{L(x)}{L(x_{opt})} \le \sqrt{3}$.

Case (II). $cov[l_0(x)] = cov[l_1(x)] =2c$. By the above discussion for Case (2), we can get that it is globally optimal to $I_2^-$.

The discussion for Case (3) is divided into two subcases given below. 

Case (I). $cov[l_0(x)] = cov[l_1(x)] =2c$. 
By the above discussion for Case (3), w.l.o.g., we assume that $\alpha_k(x) = 2$.
That is, there is a group $G_i$ with size 2 such that $\beta_i(x) = 2$. 
If there is no group with size 1, then  
\begin{eqnarray*}
\begin{aligned}
\frac{L(x)}{L(x_{opt})} \le& \frac{(|S|+1)a+ \sqrt{\frac{1-\gamma}{\gamma}(2c+(|S|+1)d)}}{|S|a+ \sqrt{\frac{1-\gamma}{\gamma}(2c+|S|d)}} \\
\le& \frac{(|S|+1)a+ \sqrt{\frac{1-\gamma}{\gamma}2c} + \sqrt{\frac{1-\gamma}{\gamma}(|S|+1)d}}{|S|a+ \sqrt{\frac{1-\gamma}{\gamma}(2c+|S|d)}} 
\le \frac{\sqrt{\frac{1-\gamma}{\gamma}4c}}{\sqrt{\frac{1-\gamma}{\gamma}2c}} \le \sqrt{2}.
\end{aligned}
\end{eqnarray*}
If there are some groups with size 1, then they are allocated to balance the number of jobs on the two machines. 
Similarly, we can also get that $\frac{L(x)}{L(x_{opt})} \le \sqrt{2}$.

Case (II). $cov[l_0(x)] + cov[l_1(x)] =2c$. W.l.o.g., assume that $cov[l_0(x)] =2c$. 
Then we have that the jobs of the groups with size 1 are all allocated to $M_1$ with respect to $x$, and we can derive that it is globally optimal to $I_2^-$. 
\end{proof}

Based on Lemma~\ref{lem:RLS for CCMSP-2^- small instance}, we are ready to analyze the approximation ratio of the local optimal solution obtained by RLS and the expected runtime for any instance of CCMSP-$2^-$.

\begin{theorem}
\label{thm:RLS for CCMSP-2^-}
Considering an instance $I_2^- = (a,c,d,\gamma,k, \{m_i | i \in [1,k]\})$ of CCMSP-$2^-$, RLS takes expected runtime $O(n^5)$ to obtain a local optimal solution to $I_2^-$ with approximation ratio $2$, where $n = \sum_{i=1}^k m_i$.
\end{theorem}
\begin{proof}
Let $x$ be a solution maintained by RLS.
W.l.o.g., assume that $x$ is not a local optimal solution, and $l'_0(x) \ge l'_1(x)$, implying $cov[l_0(x)] \ge cov[l_1(x)]$.
Observe that any mutation flipping exactly one 1-bit or two 1-bits of $x$ cannot get an improved solution. 
Thus the following discussion only considers a mutation flipping exactly a 0-bit of $G_i$ in $x$, or two 0-bits of $G_i$ and $G_j$ separately in $x$, or exactly a 0-bit of $G_i$ and a 1-bit of $G_j$ in $x$. Denote by $x'$ the obtained solution. 

Case (1). The mutation flips exactly a 0-bit of $G_i$ in $x$. For the obtained solution $x'$, we have 
\begin{align*}
    cov[l_0(x')] =& cov[l_0(x)] - 2c \left[\binom{\alpha_i(x)}{2} - \binom{\alpha_i(x)-1}{2} \right] \\
    = & cov[l_0(x)] - 2c (\alpha_i(x)-1),
\end{align*}
and 
\begin{align*}
     \ cov[l_1(x')]  = & cov[l_1(x)] - 2c \left[ \binom{\beta_i(x)}{2} - \binom{\beta_i(x)+1}{2} \right] \\
    = & cov[l_1(x)] + 2c \beta_i(x).
\end{align*}

Observe that $l'_0(x')] < l'_0(x)$.
Thus if $cov[l_1(x)] + 2c \beta_i(x) < cov[l_0(x)]$, then $L(x') < L(x)$ and $x'$ can be accepted by RLS.
If $cov[l_1(x)] + 2c \beta_i(x) = cov[l_0(x)]$ and $|M_1(x)| < |M_0(x)|$, then $L(x') \le L(x)$ and $x'$ can be accepted by RLS.

Case (2). The mutation flips two 0-bits of $G_i$ and $G_j$ separately in $x$. Observe that if $x'$ can be accepted by RLS, then the one obtained by flipping exactly a 0-bit of $G_i$ or exactly a 0-bit of $G_j$ in $x$ can be accepted by RLS as well. Thus we omit the analysis for the case.

Case (3). The mutation flips exactly a 0-bit of $G_i$ and a 1-bit of $G_j$ in $x$. For the obtained solution $x'$, we have 
\begin{eqnarray*}
\begin{aligned}
cov[l_0(x')] =& cov[l_0(x)] - \\
& ~ 2c \left[ \left(\binom{\alpha_i(x)}{2} + \binom{\alpha_j(x)}{2} \right) - \left(\binom{\alpha_i(x)-1}{2} + \binom{\alpha_j(x)+1}{2} \right) \right] \\
=& cov[l_0(x)] + 2c (\alpha_j(x) - \alpha_i(x)+1)
\end{aligned}
\end{eqnarray*}
and
\begin{eqnarray*}
\begin{aligned}
cov[l_1(x')] =& cov[l_1(x)] - \\
& ~~ 2c \left[ \left(\binom{\beta_i(x)}{2} + \binom{\beta_j(x)}{2} \right) - \left(\binom{\beta_i(x)+1}{2} + \binom{\beta_j(x)-1}{2} \right) \right] \\
=& cov[l_1(x)] + 2c (\beta_i(x) - \beta_j(x) +1).
\end{aligned}
\end{eqnarray*}
If $\alpha_j(x) \le \alpha_i(x) -1$ (i.e., $cov[l_0(x')] \le cov[l_0(x)]$) and $2c(\beta_i(x) - \beta_j(x) +1) < cov[l_0(x)] - cov[l_1(x)]$ (i.e., $cov[l_1(x')] < cov[l_0(x)]$), then $L(x') \le L(x)$ and $x'$ can be accepted by RLS; 
if $\alpha_j(x) \le \alpha_i(x) -1$ (i.e., $cov[l_0(x')] \le cov[l_0(x)]$), $2c(\beta_i(x) - \beta_j(x) +1) = cov[l_0(x)] - cov[l_1(x)]$ (i.e., $cov[l_1(x')] = cov[l_0(x)]$) and $|M_1(x)| \le |M_0(x)|$, then $L(x') \le L(x)$ and $x'$ can be accepted.

By the discussion given above for the three cases, we can get that any local optimal solution $x^*$ obtained by RLS satisfying the following two properties (assume that $l'_0(x^*) \ge l'_1(x^*)$): 
(I). For any $i \in [1,k]$ with $\alpha_i(x^*) > 0$, either $cov[l_1(x^*)] + 2c \beta_i(x^*) > cov[l_0(x^*)]$, or $cov[l_1(x^*)] + 2c \beta_i(x^*) = cov[l_0(x^*)]$ and $|M_1(x^*)| \ge |M_0(x^*)| -1$;
(II). For any $i,j \in [1,k]$ with $\alpha_i(x^*) > 0$ and $\beta_j(x^*) > 0$, either if $\alpha_j(x^*) < \alpha_i(x^*) -1$ then either $2c(\beta_i(x^*) - \beta_j(x^*) +1) > cov[l_0(x^*)] - cov[l_1(x^*)]$ or $2c(\beta_i(x^*) - \beta_j(x^*) +1) = cov[l_0(x^*)] - cov[l_1(x^*)]$ and $|M_1(x^*)| \ge |M_0(x^*)| -1$, or if $2c(\beta_i(x^*) - \beta_j(x^*) +1) < cov[l_0(x^*)] - cov[l_1(x^*)]$ or $2c(\beta_i(x^*) - \beta_j(x^*) +1) = cov[l_0(x^*)] - cov[l_1(x^*)]$ and $|M_1(x^*)| < |M_0(x^*)| -1$, then $\alpha_j(x^*) \ge \alpha_i(x^*) -1$.

Now we analyze the expected runtime of RLS to get such a local optimal solution.
For ease of analysis, let the potential function of the solution $x$ maintained by RLS during the process be $p(x) = [cov[l_{t(x)}(x)]/2c, |M_{t(x)}(x)|]$. For the solution $x'$ obtained by RLS on $x$, we say that $x'$ \emph{dominates} $x$ if $cov[l_{t(x')}(x')] < cov[l_{t(x)}(x)]$, or $cov[l_{t(x')}(x')] = cov[l_{t(x)}(x)]$ and $|M_{t(x')}(x')| \le |M_{t(x)}(x)|$; and $x'$ \emph{strongly dominates} $x$ if $x'$ dominates $x$ but $p(x') \neq p(x)$. 
Observe that $x'$ can be accepted by RLS if and only if $x'$ dominates $x$, and $x'$ is improved compared to $x$ if and only if $x'$ strongly dominates $x$.
As the number of possible values of $p(x)$ can be upper bounded by $O(n^3)$, and the probability of RLS to generate such a mutation mentioned above is $\Omega(1/n^2)$, using the Additive Drift analysis~\citep{he2004study} gets that RLS takes expected runtime $O(n^5)$ to obtain a solution $x^*$ on which any mutation of RLS cannot get an improved solution. That is, $x^*$ is a local optimal solution with respect to RLS.
 
Now we consider the approximation ratio of the obtained local optimal solution $x^*$.
As Lemma~\ref{lem:RLS for CCMSP-2^- small instance} has considered the case that $[2\sum_{i \in S} (m_i-2)(m_i+1)] - 4 m_k -8 \le 0$, where $S = \{i \in [1,k]| m_i \ge 2\}$, the following discussion considers the case that $[2\sum_{i \in S} (m_i-2)(m_i+1)] - 4 m_k -8 > 0$. 
Specifically, to meet the condition that $[2\sum_{i \in S} (m_i-2)(m_i+1)] - 4 m_k -8 > 0$, if $m_k = 3$ then $I_2^-$ contains at least three groups with size 3; if $m_k = 4$ then $I_2^-$ contains two groups with size 4 or a group with size 4 and at least one group with size 3; if $m_k \ge 5$ then no extra constraint is required.
Thus $[2\sum_{i \in S} (m_i-2)(m_i+1)] - 4 m_k -8 > 0$ implies $\sum_{i \in S} m_i(m_i-2) > 8$.

Because of the extra constraint given in Inequality~(\ref{equ:CCMSP-2^-}), our main focus is on the covariances of the loads on the two machines specified by $x^*$.
%Note that the groups with size 1 do not contribute to the covariances, thus we assume that all groups have size at least 2.  
Denote by $x_{opt}$ an optimal solution to $I^-_2$. We have
\begin{eqnarray*}
\begin{aligned}
cov[l_{t(x_{opt})}(x_{opt})] \ge& \frac{1}{2} (cov[l_{0}(x_{opt})] + cov[l_{1}(x_{opt}]) \\ 
\ge& c \sum_{i \in S} (\binom{\lfloor\frac{m_i}{2}\rfloor}{2} + \binom{\lceil\frac{m_i}{2}\rceil}{2})
\ge \frac{c}{4} \sum_{i \in S} m_i(m_i-2).
\end{aligned}
\end{eqnarray*} 
W.l.o.g., assume that $l'_0(x^*) \ge l'_1(x^*)$.

Case (1). $l'_0(x^*) = l'_1(x^*)$. 
The above discussion shows that $x^*$ is an equal solution with $cov[l_0(x^*)] = cov[l_1(x^*)]$.

Case (1.1). There exists some $t \in S$ such that $\alpha_t(x^*) = 0$.
W.l.o.g., assume that $t = \arg \max \{m_i| \alpha_i(x^*) = 0 \land i \in S\}$.
If there exists a $i \in S \setminus \{t\}$ such that $\alpha_i(x^*) > 1$ and $\beta_i(x^*) < m_t - 1$, then the solution $x^*$ can be improved by flipping a 1-bit of $G_t$ and a 0-bit of $G_i$. 
Thus for all $i \in S \setminus \{t\}$,  
if $m_i \le m_t$, then $\alpha_i(x^*) =0$ (i.e., $\beta_i(x^*) = m_i$) or $\alpha_i(x^*) =1$ (i.e., $\beta_i(x^*) = m_i -1$); 
if $m_i > m_t$ and $\alpha_i > 1$ then $\beta_i \ge m_t -1$.
Now we analyze the maximum possible value $U$ of $cov[l_0(x^*)] + cov[l_1(x^*)]$.
For ease of analysis, we replace the constraints $|M_0(x^*)| = |M_1(x^*)|$ and $cov[l_0(x^*)] = cov[l_1(x^*)]$ with $|M_0(x^*)| \ge |M_1(x^*)|$ and $cov[l_0(x^*)] \ge cov[l_1(x^*)]$. 
Let $x'$ be a solution satisfying the constraints given above that has the maximum value $U'$ of $cov[l_0(x')] + cov[l_1(x')]$. Obviously, $U' \ge U$. 
Observe that for each $i \in S$, either $\alpha_i(x') = 0$ and $\beta_i(x') = m_i$, or $\alpha_i(x') = 1$ and $\beta_i(x') = m_i -1$, or $\alpha_i(x') > 1$ and $\beta_i(x') \ge m_t -1$.
Let $S_1 = \{i \in S|\alpha_i(x') = 0 \land \beta_i(x') = m_i\}$, $S_0^1 = \{i \in S|\alpha_i(x') = 1 \land \beta_i(x') = m_i -1\}$ and $S_0^2 = S \setminus (S_1 \cup S_0^1)$.
Note that $\sum_{i \in S_0^1 \cup S_0^2} \alpha_i(x') = |S_0^1| + \sum_{i \in S_0^2} (m_i - \beta_i(x')) \ge \frac{n}{2}$.

\begin{eqnarray*}
\begin{aligned}
&\sum_{i \in S} \binom{m_i}{2} - \left(cov[l_0(x')] + cov[l_1(x')] \right)/2c \\
=& \sum_{i \in S} \binom{m_i}{2} - \\
& \quad \quad\left(\sum_{i \in S_1} \binom{m_i}{2} + \sum_{i \in S_0^1} \binom{m_i-1}{2} + \sum_{i \in S_0^2} \left(\binom{m_i - \beta_i(x')}{2} + \binom{\beta_i(x')}{2} \right) \right) \\
=& \sum_{i \in S_0^1} (m_i-1) + \sum_{i \in S_0^2} \binom{m_i}{2} - \left(\binom{m_i - \beta_i(x')}{2} + \binom{\beta_i(x')}{2} \right) \\
\ge& |S_0^1| + \sum_{i \in S_0^2} \beta_i(x') (m_i - \beta_i(x')) \ge |S_0^1| +  (m_t-1) \cdot \sum_{i \in S_0^2} (m_i - \beta_i(x')) \\ 
\ge& |S_0^1| + \sum_{i \in S_0^2} (m_i - \beta_i(x')) \ge \frac{n}{2}.
\end{aligned}
\end{eqnarray*} 

Based on the above inequality, we can get the approximation ratio of the local optimal solution $x^*$ under Case (1.1) by the following mathematical reasoning, 

\begin{eqnarray*}
\begin{aligned}
\frac{L(x^*)}{L(x_{opt})} \le& \frac{\frac{n}{2}a + \sqrt{\frac{1-\gamma}{\gamma} (cov[l_0(x^*)] + \frac{n}{2}d)}}
{\sqrt{\frac{1-\gamma}{\gamma} \frac{c}{4} \sum_{i \in S} m_i(m_i-2)}} \\
\le & \frac{\frac{n}{2}a + \sqrt{\frac{1-\gamma}{\gamma} \frac{n}{2}d} + \sqrt{\frac{1-\gamma}{\gamma} cov[l_0(x^*)]}}
{\sqrt{\frac{1-\gamma}{\gamma} \frac{c}{4} \sum_{i \in S} m_i(m_i-2)}} \\
\le& \frac{\sqrt{\frac{1-\gamma}{\gamma} (cov[l_0(x^*)]+2c)}}
{\sqrt{\frac{1-\gamma}{\gamma} \frac{c}{4} \sum_{i \in S} m_i(m_i-2)}}
\le \sqrt{\frac{cov[l_0(x^*)]+2c}{\frac{c}{4} \sum_{i \in S} m_i(m_i-2)}} \\
\le& \sqrt{\frac{c (\sum_{i \in S} \binom{m_i}{2} - \frac{n}{2})+ 2c}{\frac{c}{4} \sum_{i \in S} m_i(m_i-2)}} 
= \sqrt{\frac{4\sum_{i \in S} \binom{m_i}{2} -2n+8}{\sum_{i \in S} m_i(m_i-2)}} \\
=& \sqrt{\frac{2\sum_{i \in S} m_i(m_i-2)+8}{\sum_{i \in S} m_i(m_i-2)}} < \sqrt{3},
\end{aligned}
\end{eqnarray*} 
where the last inequality holds as $\sum_{i \in S} m_i(m_i-2) > 8$.

Case (1.2). There is no $i \in S$ such that $\alpha_i(x^*) = 0$.
\begin{equation}
\label{equ:no 0}
\begin{aligned}
    cov[l_0(x^*)]+cov[l_1(x^*)] &= 2c \sum_{i \in S} (\binom{\alpha_i}{2} + \binom{\beta_i}{2}) \\
    & \leq 2c \sum_{i \in S} \left(\binom{m_i-1}{2} + \binom{1}{2} \right) \\
    &\quad = c \sum_{i \in S} (m_i-1)(m_i-2).
\end{aligned}
\end{equation}
Then we can get the approximation ratio of the solution $x^*$ under Case (1.2) by the following mathematical reasoning, 
\begin{eqnarray*}
\begin{aligned}
\frac{L(x^*)}{L(x_{opt})} \le& \frac{\frac{n}{2}a + \sqrt{\frac{1-\gamma}{\gamma} (cov[l_0(x^*)] + \frac{n}{2}d)}}
{\sqrt{\frac{1-\gamma}{\gamma} \frac{c}{4} \sum_{i \in S} m_i(m_i-2)}}
\le \frac{\sqrt{cov[l_0(x^*)]+2c}}
{\sqrt{\frac{c}{4} \sum_{i \in S} m_i(m_i-2)}} \\
\le& \sqrt{\frac{\frac{c \sum_{i \in S} (m_i-1)(m_i-2)}{2}+2c}
{\frac{c}{4} \sum_{i \in S} m_i(m_i-2)}} 
= \sqrt{\frac{2 \sum_{i \in S} (m_i-1)(m_i-2)+8}
{\sum_{i \in S} m_i(m_i-2)}}  \\
\le& \sqrt{\frac{2\sum_{i \in S} m_i(m_i-2) - 2n+4k+8}
{\sum_{i \in S} m_i(m_i-2)}} \\
<& \sqrt{\frac{2\sum_{i \in S} m_i(m_i-2) +8}
{\sum_{i \in S} m_i(m_i-2)}} < \sqrt{3}, 
\end{aligned}
\end{eqnarray*} 
where the last inequality holds as $\sum_{i \in S} m_i(m_i-2) > 8$.

Case (2). $l'_0(x^*) > l'_1(x^*)$.

Case (2.1). $cov[l_0(x^*)] = cov[l_1(x^*)]$ and $|M_0(x^*)| > |M_1(x^*)|$. For the case, using the reasoning similar to that given for Case (1), we can derive the same conclusion.

Case (2.2). $cov[l_0(x^*)] > cov[l_1(x^*)]$.
Let $\delta = (cov[l_0(x^*)] - cov[l_1(x^*)])/2c$.

Case (2.2.1). There exists some $t \in [1,k]$ such that $\alpha_t = 0$.

W.l.o.g., assume that $t = \arg \max \{m_i| \alpha_i(x^*) = 0 \land i \in S\}$.
If there exists a $i \in S \setminus \{t\}$ such that $\alpha_i(x^*) > 1$ and $\beta_i(x^*) < m_t - 1 + \delta$, then the solution $x^*$ can be improved by flipping a 0-bit of $G_i$ and a 1-bit of $G_t$. 
Thus for all $i \in S$, if $m_i \le m_t$, then $\alpha_i(x^*) =0$ (i.e., $\beta_i(x^*) = m_i$) or $\alpha_i(x^*) =1$ (i.e., $\beta_i(x^*) = m_i -1$); if $m_i > m_t$ and $\alpha_i(x^*) > 1$ then $\beta_i(x^*) \ge m_t -1 + \delta \ge 2$.
Observe that $\delta$ cannot over $m_k$; otherwise, the inequality $\beta_k(x^*) \ge m_t -1 + \delta$ cannot hold as $m_t \ge 2$ and $\beta_k(x^*) \le m_k$.  

Now we analyze the maximum value $U$ of $cov[l_0(x^*)] + cov[l_1(x^*)]$.
Let $x'$ be a solution satisfying the constraints given above that has the maximum value of $cov[l_0(x')] + cov[l_1(x')]$.
Observe that for each $i \in S$, either $\alpha_i(x') = 0$ or $\alpha_i(x') > 0$.
Let $S_1 = \{i|\alpha_i(x') = 0 \land \beta_i(x') = m_i\}$, $S_0^1 = \{i|\alpha_i(x') = 1 \land \beta_i(x') = m_i -1\}$ and $S_0^2 = S \setminus (S_1 \cup S_0^1)$.

\begin{eqnarray*}
\begin{aligned}
&\left(cov[l_0(x')] + cov[l_1(x')] \right)/2c \\ 
=& \sum_{i \in S_1} \binom{m_i}{2} + \sum_{i \in S_0^1} \binom{m_i-1}{2} + \sum_{i \in S_0^2} \left(\binom{\alpha_i(x')}{2} + \binom{m_i -\alpha_i(x')}{2} \right) \\
=& \sum_{i \in S} \binom{m_i}{2} - \\
& ~ \sum_{i \in S_0^1} \left(\binom{m_i}{2} - \binom{m_i-1}{2} \right) - \sum_{i \in S_0^2} \left(\binom{m_i}{2} - \binom{\alpha_i(x')}{2} - \binom{m_i -\alpha_i(x')}{2} \right) \\
=& \sum_{i \in S} \binom{m_i}{2} - \sum_{i \in S_0^1} \left(m_i-1 \right) - \sum_{i \in S_0^2} \left(\binom{m_i}{2} - \binom{\alpha_i(x')}{2} - \binom{m_i -\alpha_i(x')}{2} \right) \\
=& \sum_{i \in S} \frac{m_i(m_i-1)}{2} - \sum_{i \in S_0^1} \left(m_i-1 \right) - \sum_{i \in S_0^2} (m_i - \alpha_i)\alpha_i \\ 
=& \sum_{i \in S} \frac{m_i(m_i-1)}{2} - \sum_{i \in S_0^1 \cup S_0^2} (m_i - \alpha_i)\alpha_i.
\end{aligned}
\end{eqnarray*}

Based on the above inequality, we get the approximation ratio by the following mathematical reasoning, in which $n' = \sum_{i \in S} m_i$. 
\begin{equation*}
\begin{aligned}
&\frac{L(x^*)}{L(x_{opt})} \le \frac{\sqrt{\frac{1-\gamma}{\gamma} (cov[l_0(x^*)] + 2c)}}
{\sqrt{\frac{1-\gamma}{\gamma} \frac{c}{4} \sum_{i=1}^{k} m_i(m_i-2)}} 
\le \frac{\sqrt{cov[l_0(x^*)]+2c}}
{\sqrt{\frac{c}{4} \sum_{i \in S} m_i(m_i-2)}} \\
&\le \sqrt{\frac{\frac{1}{2}\sum_{i \in S} m_i(m_i-1) - \sum_{i \in S_0^1 \cup S_0^2} (m_i - \alpha_i)\alpha_i + \delta +2}
{\frac{1}{4} \sum_{i \in S} m_i(m_i-2)}} \\
&\leq \sqrt{2} \cdot \sqrt{\frac{\sum_{i \in S} m_i(m_i-1) - 2\sum_{i\in S_0^1 \cup S_0^2} (m_i-\alpha_i)\alpha_i + 2\delta +4}{\sum_{i \in S} m_i(m_i-2)}} \\
&\leq \sqrt{2} \cdot \sqrt{1 + \frac{n' + 2\delta - \sum_{i\in S_0^1 \cup S_0^2} 2(m_i-\alpha_i)\alpha_i +4}{\sum_{i \in S} m_i(m_i-2)}} \\
&\leq \sqrt{2} \cdot \sqrt{1 + \frac{\sum_{i\in S_0^1 \cup S_0^2} (-m_i^2 + (2\alpha_i+1)m_i - 2\alpha_i) +A}{\sum_{i \in S} m_i(m_i-2)}} \\
&\leq \sqrt{2} \cdot \sqrt{1+ \frac{\sum_{i\in S_0^1 \cup S_0^2}(- m_i^2+m_i+2\alpha_i^2 - 2\alpha_i)  - \sum_{i\in S_1}(m_i^2-m_i) +4}{\sum_{i \in S} m_i(m_i-2)}} \\
&\leq \sqrt{2} \cdot \sqrt{1+ \frac{- \sum_{i \in S} (m_i^2-2m_i) + \sum_{i\in S_0^1 \cup S_0^2}(2\alpha_i^2 - 2\alpha_i) +4}{\sum_{i \in S} m_i(m_i-2)}} \\
&\leq \sqrt{2} \cdot \sqrt{\frac{\sum_{i\in S_0^1 \cup S_0^2}2\alpha_i(\alpha_i - 1) +4}{\sum_{i \in S} m_i(m_i-2)}} 
= \sqrt{2} \cdot \sqrt{\frac{\sum_{i\in S_0^2}2\alpha_i(\alpha_i - 1) +4}{\sum_{i \in S} m_i(m_i-2)}} \\
&\leq \sqrt{2} \cdot \sqrt{\frac{\sum_{i\in S_0^2}2\alpha_i(\alpha_i - 1) +4}{\sum_{i \in S_0^2} m_i(m_i-2)}} 
\leq \sqrt{2} \cdot \sqrt{\frac{\sum_{i\in S_0^2}2\alpha_i(\alpha_i - 1) +4}{\sum_{i \in S_0^2} (\alpha_i +2)\alpha_i}} < 2,
\end{aligned}
\end{equation*}
where $A= - \sum_{i\in S_1}  m_i(m_i-1) - \sum_{i\in S_0^1 \cup S_0^2} 2(m_i-\alpha_i)\alpha_i + n' +4$, and the second to the last inequality holds as $\beta_i(x^*) \ge 2$ for each $i \in S_0^2$.

Case (2.2.2). There is no $i \in [1,k]$ such that $\alpha_i(x^*) = 0$. 
By Inequality~(\ref{equ:no 0}), we get the approximation ratio as the following mathematical reasoning, 

\begin{eqnarray*}
\begin{aligned}
\frac{L(x^*)}{L(x_{opt})} \le& \frac{\sqrt{\frac{1-\gamma}{\gamma} (cov[l_0(x^*)] + 2c)}}
{\sqrt{\frac{1-\gamma}{\gamma} \frac{c}{4} \sum_{i \in S} m_i(m_i-2)}}
\le \frac{\sqrt{cov[l_0(x^*)]+2c}}
{\sqrt{\frac{c}{4} \sum_{i \in S} m_i(m_i-2)}} \\
\le& \sqrt{\frac{\frac{\sum_{i \in S} (m_i-1)(m_i-2)}{2} + \delta +2}
{\frac{1}{4} \sum_{i \in S} m_i(m_i-2)}} 
= \sqrt{\frac{2 \sum_{i \in S} (m_i-1)(m_i-2)+ 4\delta +8}
{\sum_{i \in S} m_i(m_i-2)}}  \\
\le& \sqrt{\frac{2\sum_{i \in S} m_i(m_i-2) - 2n'+4|S|+4\delta+8}
{\sum_{i \in S} m_i(m_i-2)}}\\ 
\le& \sqrt{2 + \frac{-2n'+4|S| +4 m_k+8}
{\sum_{i \in S} m_i(m_i-2)}} < 2, 
\end{aligned}
\end{eqnarray*}
where the last inequality holds because of the condition $[2\sum_{i \in S} (m_i-2)(m_i+1)] - 4 m_k -8 > 0$ that we have mentioned above.
\end{proof}

\section{Experiment}

% \xiankun{reviewer 4 - avoid using asterisks $(*)$ for denoting multiplication
% experiments: as mentioned above, the idea of capping the number of iterations from above by $10^5$ and presenting the results as if the runtime cannot be greater than $10^5$ is not a very good idea. For easy cases, the runtimes just need to be computed fairly. For harder cases, this should be done in a nice way (e.g. by considering different upper limits for different algorithms and problems and presenting the full picture, or by indicating how many runs have been terminated in time).}

The section studies the experimental performance of RLS and (1+1)~EA for the three variants CCMSP-1, CCMSP-$2^+$ and CCMSP-$2^-$.
The values of the common variables considered in the three variants are set as follows: The expected processing time of each job is $a = 100$ and its variance is $d = 10^{-2}$; the acceptable threshold is $\gamma = 0.05$; the number of groups is $k \in \{2^2,2^3,2^4,2^5,2^6,2^7\}$.
Recall that all jobs considered in the three variants have the same expected processing time $a$ and variance $d$, and the related expected runtime and properties of the local optimal solutions mentioned above do not depend on $a$ and $d$. Thus we just consider one setting approach for the values of $a$ and $d$ in our experimental work.

\emph{CCMSP-1}. The number of jobs in the group is set as $m \in \{10,50,100,200,\\ 300,400\}$. 
For the value of the covariance, we consider the three setting approaches given below.
Firstly, if the jobs in the same group are the same, then we have $cov(X,X) = var(X) = 10^{-2}$, where $X$ denotes a job. Thus the first setting approach is (1). $c = 10^{-2}$.
For the other case that the jobs in the same group are not the same, we consider the two setting approaches: (2). $c = 10^{-7}$ and (3). $c = 10^{20}$, to analyze whether the extra constraint considered in CCMSP-$2^+$ and the one considered in CCMSP-$2^-$ have an impact on the experimental performance of the two algorithms for the instances of CCMSP-1, where $c = 10^{-7}$ and $c = 10^{20}$ meet the extra constraints of CCMSP-$2^+$ and CCMSP-$2^-$, respectively.

\begin{figure}[h]
\centering
\begin{minipage}{0.47\linewidth}
\centering
\includegraphics[width=1\linewidth]{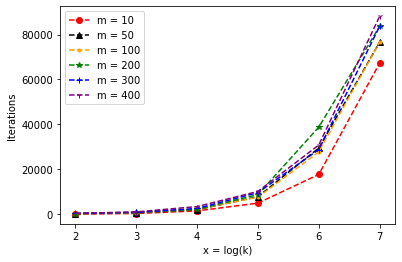}
\caption{RLS for CCMSP-1 with $c = 10^{-2}$.}
\label{fig:experiment 1}
\end{minipage}
\begin{minipage}{0.47\linewidth}
\centering
\includegraphics[width=1\linewidth]{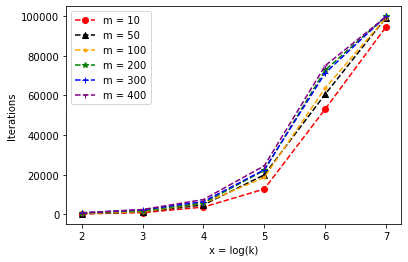}
\caption{(1+1) EA for CCMSP-1 with $c = 10^{-2}$.}
\label{fig:experiment 2}
\end{minipage}
% \qquad
% \begin{minipage}{0.47\linewidth}
% \centering
% \includegraphics[width=1\linewidth]{imgs/CCMSP1/CCMSP1RLS_10^-3.png}
% \caption{RLS for CCMSP-1 with $c = 10^{-3}$.}
% \label{fig:experiment 3}
% \end{minipage}
% \begin{minipage}{0.47\linewidth}
% \centering
% \includegraphics[width=1\linewidth]{imgs/CCMSP1/CCMSP1EA_10^-3.png}
% \caption{(1+1) EA for CCMSP-1 with $c = 10^{-3}$.}
% \label{fig:experiment 4}
% \end{minipage}
\qquad
\begin{minipage}{0.47\linewidth}
\centering
\includegraphics[width=1\linewidth]{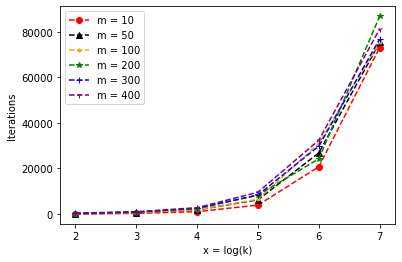}
\caption{RLS for CCMSP-1 with $c = 10^{-7}$.}
\label{fig:experiment 3}
\end{minipage}
\begin{minipage}{0.47\linewidth}
\centering
\includegraphics[width=1\linewidth]{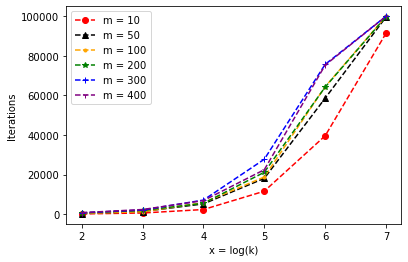}
\caption{(1+1) EA for CCMSP-1 with $c = 10^{-7}$.}
\label{fig:experiment 4}
\end{minipage}
\qquad
\begin{minipage}{0.47\linewidth}
\centering
\includegraphics[width=1\linewidth]{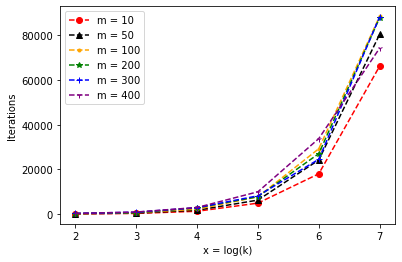}
\caption{RLS for CCMSP-1 with $c = 10^{20}$.}
\label{fig:experiment 5}
\end{minipage}
\begin{minipage}{0.47\linewidth}
\centering
\includegraphics[width=1\linewidth]{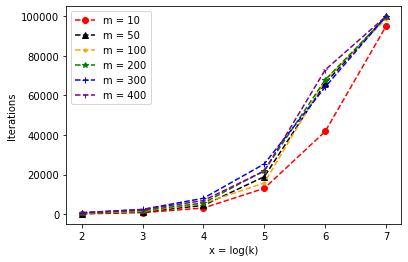}
\caption{(1+1) EA for CCMSP-1 with $c = 10^{20}$.}
\label{fig:experiment 6}
\end{minipage}
\vspace*{-5mm}
\end{figure}

For each constructed instance of CCMSP-1, we execute each of the two algorithms on it for 30 times and take the average of the iteration numbers required for the 30 executions to solve it. 
However, for some instances, some executions of the algorithm cannot solve them within $10^5$ iterations. Under the case, we set $10^5$ as the iteration number that the algorithm requires to solve it. 
Taking into account the randomness of the experimental performance of the two algorithms, by Fig.~\ref{fig:experiment 1}-\ref{fig:experiment 6} we find that: 
(1). The value of the co-variance has little influence on the performance of the two algorithms for the instances of CCMSP-1, so do the extra constraints given in CCMSP-$2^+$ and CCMSP-$2^-$; 
(2). For each constructed instance of CCMSP-1, the averaged iteration number of the (1+1) EA for solving it is always larger than that of RLS, coinciding with the obtained theoretical performance of the two algorithms for CCMSP-1; 
and (3). The six curves increase exponentially with $\log k$, coinciding with the obtained theoretical performance of the two algorithms for CCMSP-1. 
Remark that for Fig.~\ref{fig:experiment 2},~\ref{fig:experiment 4} and~\ref{fig:experiment 6}, the slopes of some curves in region $[6,7]$ decrease compared to that in region $[5,6]$, as the (1+1) EA cannot solve these instances within $10^5$ iterations so that $10^5$ is set as the iteration number that the algorithm requires to solve them. The real slopes should be larger than that depicted in region $[6,7]$ of Fig.~\ref{fig:experiment 2},~\ref{fig:experiment 4} and~\ref{fig:experiment 6}.

\emph{CCMSP-$2^+$}. We separately investigate the experimental performance of RLS for CCMSP-$2^+$-Even and CCMSP-$2^+$-Odd.
Because of the extra constraint considered in CCMSP-$2^+$, we consider the three setting approaches of the covariance $c \in \{10^{-7}, 10^{-8}, 10^{-9}\}$.
The number of jobs considered in each instance is $n \in \{k \times 10,k \times 50,k \times 100,k \times 200,k \times 300,k \times 400\}$ (recall that $k \in \{2^2,2^3,2^4,2^5,2^6,2^7\}$). 
That is, for each instance constructed for CCMSP-$2^+$-Even, it has a parameter-triple $(k,n,c)$.
For each specific parameter-triple $(k,n,c)$, each $m_i$ ($1 \le i \le k$) is set in order as follows: If $1 \le i \le k-1$ then $m_i = f + g$, where $f = \lceil(n - \sum_{j=1}^{i-1} m_i)/(k+1-i) \rceil$ and $g$ is an integer chosen from the range $[-\lceil f/2\rceil, \lceil f/2\rceil]$ uniformly at random; otherwise, $m_k = n - \sum_{j=1}^{k-1} m_i$.
Therefore, there are totally 108 instances totally constructed for CCMSP-$2^+$-Even.
To compare the experimental performance of RLS for CCMSP-$2^+$-Even and CCMSP-$2^+$-Odd, for each one $I$ of the 108 instances constructed above for CCMSP-$2^+$-Even, an extra job is added into a group uniformly at random such that an instance $I'$ of CCMSP-$2^+$-Odd is constructed. That is, there are 108 instances totally constructed for CCMSP-$2^+$-Odd, and each instance $I'$ of CCMSP-$2^+$-Odd has an extra job compared to the corresponding instance $I$ of CCMSP-$2^+$-Even.

Remark that for each constructed instance of CCMSP-$2^+$-Even, RLS would stop once a solution satisfying the conditions given in Theorem~\ref{thm:RLS for CCMSP-2^+ even} is obtained (not optimal). However, by Fig.~\ref{fig:experiment 7},~\ref{fig:experiment 9} and ~\ref{fig:experiment 11}, almost all instances of CCMSP-$2^+$-Even cannot be solved by RLS within $10^5$ iterations if $k \ge 2^3$. 
In comparison, the experimental performance of RLS for the instances of CCMSP-$2^+$-Odd to obtain the optimal solutions, given in Fig.~\ref{fig:experiment 8},~\ref{fig:experiment 10} and~\ref{fig:experiment 12}, is much better, coinciding with the obtained theoretical performance of RLS for CCMSP-$2^+$. 

Specifically, for the setting $c = 10^{-7}$, the makespan of the obtained solutions for the corresponding 36 instances of CCMSP-$2^+$-Even and the 36 ones of CCMSP-$2^+$-Odd are listed in Table~\ref{tab:my-table}.
It is necessary to point out that each instance $I'$ of CCMSP-$2^+$-Odd has an extra job compared to the corresponding instance $I$ of CCMSP-$2^+$-Even, thus the makespan of the optimal solution for $I'$ is at least 100 more than that of $I$, taking into account the expected runtime $a = 100$ of the extra job, and the resulting variance and covariance.
However, as mentioned before, the obtained solution for $I'$ is not optimal but the one for $I$ is, thus the gap between the two makespan is around 100 (i.e., the gap may be larger or less than 100).

\begin{table}[h]
\resizebox{\textwidth}{!}{%
\begin{tabular}{|l|rr|rr|rr|rr|rr|rr|}
\hline
\multirow{2}{*}{} & \multicolumn{2}{c|}{$k=2^2$}                                   & \multicolumn{2}{c|}{$k=2^3$}                                    & \multicolumn{2}{c|}{$k=2^4$}                                   & \multicolumn{2}{c|}{$k=2^5$}                                   & \multicolumn{2}{c|}{$k=2^6$}                                   & \multicolumn{2}{c|}{$k=2^7$}                                  \\ \cline{2-13} 
                  & \multicolumn{1}{c|}{Even}       & \multicolumn{1}{c|}{Odd} & \multicolumn{1}{c|}{Even}        & \multicolumn{1}{c|}{Odd} & \multicolumn{1}{c|}{Even}        & \multicolumn{1}{c|}{Odd} & \multicolumn{1}{c|}{Even}        & \multicolumn{1}{c|}{Odd} & \multicolumn{1}{c|}{Even}        & \multicolumn{1}{c|}{Odd} & \multicolumn{1}{c|}{Even}        & \multicolumn{1}{c|}{Odd} \\ \hline
$n=k*10 \ / \ k*10+1$   & \multicolumn{1}{r|}{2,001.949,4}  & 2,101.997,5                & \multicolumn{1}{r|}{4,002.756,9}   & 4,102.791,1                & \multicolumn{1}{r|}{8,003.898,9}   & 8,103.923,1                & \multicolumn{1}{r|}{16,005.514,3}  & 16,105.530,9               & \multicolumn{1}{r|}{32,007.799,2}  & 32,107.809,8               & \multicolumn{1}{r|}{64,011.032,6}  & 64,111.036,1               \\ \hline
$n=k*50 \ / \ k*50+1$   & \multicolumn{1}{r|}{10,004.359,7} & 10,104.381,2               & \multicolumn{1}{r|}{20,006.165,6}  & 20,106.180,6               & \multicolumn{1}{r|}{40,008.720,5}  & 40,108.729,8               & \multicolumn{1}{r|}{80,012.335,7}  & 80,112.338,1               & \multicolumn{1}{r|}{160,017.453,9} & 160,117.443,2              & \multicolumn{1}{r|}{320,024.711,8} & 320,124.664,7              \\ \hline
$n=k*100 \ / \ k*100+1$ & \multicolumn{1}{r|}{20,006.166,5} & 20,106.181,3               & \multicolumn{1}{r|}{40,008.720,9}  & 40,108.730,9               & \multicolumn{1}{r|}{80,012.335,6}  & 80,112.339,6               & \multicolumn{1}{r|}{160,017.453,8} & 160,117.445,5              & \multicolumn{1}{r|}{320,024.713,6} & 320,124.667,8              & \multicolumn{1}{r|}{640,035.015,9} & 640,134.886               \\ \hline
$n=k*200 \ / \ k*200+1$ & \multicolumn{1}{r|}{40,008.724,5} & 40,108.733,1               & \multicolumn{1}{r|}{80,012.337,9}  & 80,112.342,7               & \multicolumn{1}{r|}{160,017.453,4} & 160,117.449,9              & \multicolumn{1}{r|}{320,024.715,7} & 320,124.674,2              & \multicolumn{1}{r|}{640,035.019,3} & 640,134.894,6              & \multicolumn{1}{r|}{1,280,049.693} & 1,280,149.384              \\ \hline
$n=k*300 \ / \ k*300+1$ & \multicolumn{1}{r|}{60,010.688,2} & 60,110.694,1               & \multicolumn{1}{r|}{120,015.115,3} & 120,115.117,2              & \multicolumn{1}{r|}{240,021.389,9} & 240,121.374,8              & \multicolumn{1}{r|}{480,030.293,7} & 480,130.225,6              & \multicolumn{1}{r|}{960,042.980,6} & 960,142.758,5              & \multicolumn{1}{r|}{1,920,061.155} & 1,920,160.615              \\ \hline
$n=k*400 \ / \ k*400+1$ & \multicolumn{1}{r|}{80,012.346,4} & 80,112.348,8               & \multicolumn{1}{r|}{160,017.460,7} & 160,117.459,4              & \multicolumn{1}{r|}{320,024.712}  & 320,124.686,3              & \multicolumn{1}{r|}{640,035.024,7} & 640,134.910,6              & \multicolumn{1}{r|}{1,280,049.759} & 1,280,149.414              & \multicolumn{1}{r|}{2,560,070.952} & 2,560,170.201              \\ \hline
\end{tabular}%
}
\caption{The experimental results about the makespan for the 36 instances of CCMSP-$2^+$-Even and the 36 ones of CCMSP-$2^+$-Odd with $c = 10^{-7}$.}
\label{tab:my-table}
\end{table}

\emph{CCMSP-$2^-$}. 
To meet the extra constraint considered in CCMSP-$2^-$, we consider the three setting approaches of the covariance $c \in \{10^{20}, 10^{25}, 10^{30}\}$. 
The way to construct the instances of CCMSP-$2^-$ is the same as that of CCMSP-$2^+$-Even, i.e., there are totally 108 instances constructed as well.
Remark that for each constructed instance of CCMSP-$2^-$, RLS stops once a local optimal solution satisfying the conditions given in Theorem~\ref{thm:RLS for CCMSP-2^-} is obtained (not optimal).
By Fig.~\ref{fig:experiment 12}, ~\ref{fig:experiment 13} and~\ref{fig:experiment 14}, we have that if the value of $c$ is large enough, then its influence on the performance of the algorithm is little.
Specifically, the curves given in the three figures are almost the same if $k \ge 2^4$. 
Besides, we also study the optimization process specified by an execution of RLS on an instance of CCMSP-$2^-$ with $c = 10^{20}$, $k = 2^7$ and $n = k*300$ (recall that $n$ is even), which is given in Fig.~\ref{fig:experiment 16}. 
Note that the execution would not stop unless the iteration number reaches $10^5$, and the vertical coordinate of Fig.~\ref{fig:experiment 16} is $(\frac{L(x)}{LB} -1) \cdot 1000$, where $LB = \sqrt{\frac{1-\gamma}{\gamma} \cdot \frac{c}{4} \cdot \sum_{i \in S} m_i(m_i-2)}$ is the lower bound of $L(x^*)$ of the optimal solution $x^*$ (we use the lower bound $LB$ to approximately replace $L(x^*)$ as it is computationally expensive to get $L(x^*)$, and $\frac{L(x)}{LB}$ is larger than the real approximation ratio).  
We find that the optimization speed is really fast at the early stage and the $y$-value is improved to 1.5 within the first 150 iterations, then the optimization speed gradually declines, and finally the $y$-value approaches 0.2 when the iteration number reaches $10^5$.

\begin{figure}[h]
\centering
\begin{minipage}{0.47\linewidth}
\centering
\includegraphics[width=1.05\linewidth]{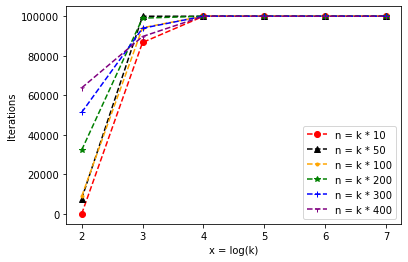}
\caption{RLS for CCMSP-$2^+$-Even with $c = 10^{-7}$.}
\label{fig:experiment 7}
\end{minipage}
\begin{minipage}{0.47\linewidth}
\centering
\includegraphics[width=1.05\linewidth]{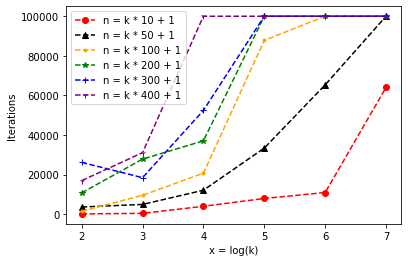}
\caption{RLS for CCMSP-$2^+$-Odd with $c = 10^{-7}$.}
\label{fig:experiment 8}
\end{minipage}
\qquad
\begin{minipage}{0.47\linewidth}
\centering
\includegraphics[width=1.05\linewidth]{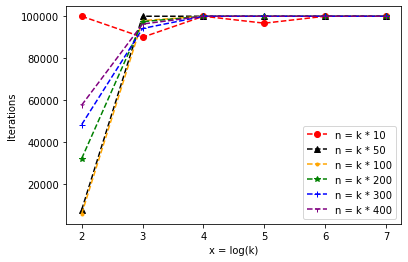}
\caption{RLS for CCMSP-$2^+$-Even with $c = 10^{-8}$.}
\label{fig:experiment 9}
\end{minipage}
\begin{minipage}{0.47\linewidth}
\centering
\includegraphics[width=1.05\linewidth]{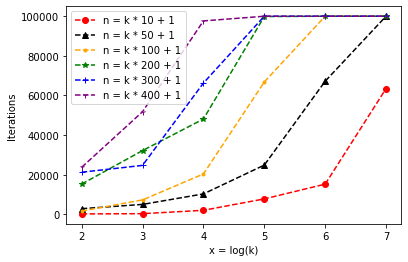}
\caption{RLS for CCMSP-$2^+$-Odd with $c = 10^{-8}$.}
\label{fig:experiment 10}
\end{minipage}
\qquad
\begin{minipage}{0.47\linewidth}
\centering
\includegraphics[width=1.05\linewidth]{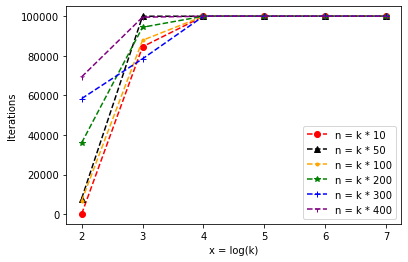}
\caption{RLS for CCMSP-$2^+$-Even with $c = 10^{-9}$.}
\label{fig:experiment 11}
\end{minipage}
\begin{minipage}{0.47\linewidth}
\centering
\includegraphics[width=1.05\linewidth]{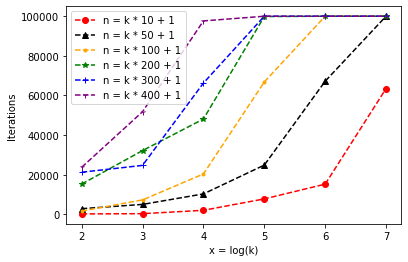}
\caption{RLS for CCMSP-$2^+$-Odd with $c = 10^{-9}$.}
\label{fig:experiment 12}
\end{minipage}
\end{figure}

\begin{figure}[h]
\centering
\begin{minipage}{0.47\linewidth}
\centering
\includegraphics[width=1.05\linewidth]{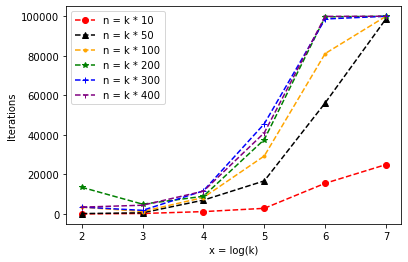}
\caption{RLS for CCMSP-$2^-$ with $c = 10^{20}$.}
\label{fig:experiment 13}
\end{minipage}
\begin{minipage}{0.47\linewidth}
\centering
\includegraphics[width=1.05\linewidth]{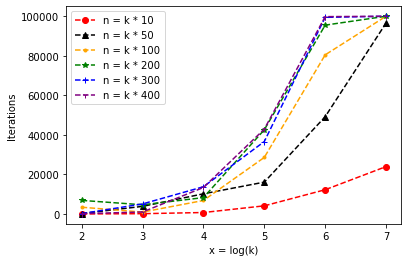}
\caption{RLS for CCMSP-$2^-$ with $c = 10^{25}$.}
\label{fig:experiment 14}
\end{minipage}
\qquad
\begin{minipage}{0.47\linewidth}
\centering
\includegraphics[width=1.05\linewidth]{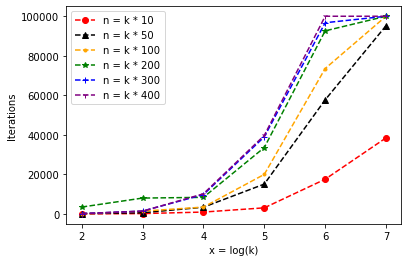}
\caption{RLS for CCMSP-$2^-$ with $c = 10^{30}$.}
\label{fig:experiment 15}
\end{minipage}
\begin{minipage}{0.47\linewidth}
\centering
\includegraphics[width=0.9 \linewidth]{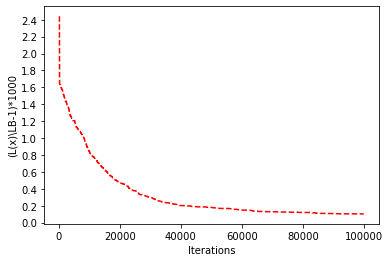}
\caption{The optimization process of RLS on an instance of CCMSP-$2^-$ with $c = 10^{20}$, $k = 2^7$ and $n = k * 300$.}
\label{fig:experiment 16}
\end{minipage}
\end{figure}

\section{Conclusion}

The paper studied a chance-constrained version of the makespan scheduling problem and investigated the performance of RLS and the (1+1) EA for it.
More specifically, the paper studied three simple variants of the problem (namely, CCMSP-1, CCMSP-2$^+$ and CCMSP-2$^-$) and obtained a series of results: CCMSP-1 was shown to be polynomial-time solvable by giving the expected runtime  of RLS and the (1+1) EA (i.e., $O(n^2 / m = O(kn))$ and $O((k+m)n^2)$, respectively) to obtain an optimal solution to the given instance of CCMP-1; CCMSP-2$^+$ was shown to be NP-hard by reducing the Two-way Balanced Partition problem to it,
while any instance of CCMSP-2$^+$ that considers odd number of  jobs was shown to be polynomial-time solvable by giving the expected runtime of RLS (i.e., $O(\sqrt{k}n^3)$) to obtain an optimal solution to it, and the properties and approximation ratio of the local optimal solution obtained by RLS for any instance of CCMSP-2$^+$ that considers even number of  jobs were given;
CCMSP-2$^-$ was shown to be NP-hard as well, and the discussion for the approximation ratio 2 of the local optimal solution obtained by RLS was given.
Related experimental results were found to coincide with these theoretical results.

It is necessary to point out that the study of RLS and the (1+1) EA for CCMSP given in the paper, compared to that for MSP given in the previous related literatures, is not just considering a new surrogate of the makespan induced by One-sided Chebyshev’s Inequality.
It can be found that the value of the surrogate is also affected by the variances among the jobs of the same group assigned to the same machine. 
Thus we have to not only consider the expected processing time and variance of each job, but also the covariances. 
In some way, the covariances can be treated as weights on the jobs, but they do not stay the same and dynamically change with the number of jobs of the same group assigned to the same machine. 

To our best knowledge, except the literature~\cite{xie2019evolutionary}, there is no one in the area of theoretical study of EA, in which the weights on the bits are non-linearly changed with the values of its neighborhood bits. Thus from this perspective, our study and the one given in~\citep{xie2019evolutionary} are the stepping stones to the new research area. 

Future work on the other variants of the chance-constrained makespan scheduling problem would be interesting, e.g., the variant of CCMSP-1 in which the covariance has two or more optional values, and the one of CCMSP in which everything is unknown except the loads on the two machines. 
Additionally, studying the chance-constrained version of other classical combinatorial optimization problems is also meaningful.
These related results would further advance and broaden the understanding of evolutionary algorithms.

\section{Acknowledgements}

This work has been supported by the National Natural Science Foundation of China under Grants 62472449, the Open Project of Xiangjiang Laboratory under Grant 22XJ03005, and the Australian Research Council (ARC) through grant FT200100536.

\small



 \bibliographystyle{elsarticle-num} 
 \bibliography{ref}





% \end{thebibliography}
\end{document}